%% file: main.tex
\documentclass[10pt,twocolumn,letterpaper]{article}

\usepackage[pagenumbers]{cvpr} 
\input{preamble}
\definecolor{cvprblue}{rgb}{0.21,0.49,0.74}
\usepackage[pagebackref,breaklinks,colorlinks,citecolor=cvprblue]{hyperref}

\title{Quantum Multiple Rotation Averaging}

\author{
Shuteng Wang\textsuperscript{1} \quad
Natacha Kuete Meli\textsuperscript{2} \quad
Michael Möller\textsuperscript{2} \quad
Vladislav Golyanik\textsuperscript{1} \\[1ex]
\textsuperscript{1}Max Planck Institute for Informatics, SIC\quad\quad
\textsuperscript{2}University of Siegen
}

\begin{document}
\maketitle
\input{sec/abstract}    
\input{sec/paper}

\noindent \textbf{Acknowledgement.} 
This work was supported by the Deutsche Forschungsgemeinschaft (DFG, German Research Foundation), project number 534951134.
NKM and MM acknowledge support by the Lamarr Institute for Machine Learning and Artificial Intelligence.

\newpage

{
    \small
    \bibliographystyle{ieeenat_fullname}
    \bibliography{main}
}
\input{sec/suppl}

\end{document}

%% file: preamble.tex
\usepackage[dvipsnames]{xcolor}

\usepackage{amsmath,amssymb,amsthm,times}
\usepackage{bbm}
\usepackage{graphicx}
\usepackage{multicol} 
\usepackage{multirow}
\usepackage{booktabs}
\usepackage{algorithm}
\usepackage{algpseudocode}
\usepackage{braket}
\usepackage{wrapfig}
\usepackage{times}
\usepackage{epsfig}

\newtheorem{theorem}{Theorem}

\newtheorem{proposition}{Proposition}

\newtheoremstyle{boldremark}
  {\topsep}    
  {\topsep}    
  {\normalfont} 
  {}           
  {\bfseries}  
  {.}          
  { }          
  {}           

\theoremstyle{boldremark}
\newtheorem{remark}{Remark}

\newcommand{\vect}{\mathrm{vec}}
\newcommand{\SO}{\textsc{SO}}

\newcommand{\R}{\mathbb{R}}
\newcommand{\B}{\mathbb{B}}
\newcommand{\vit}{\mathit{v}}

\DeclareMathOperator*{\argmin}{arg\,min}
\usepackage{bm}

\newcommand{\Matrix}[2][0.8]{%
	\left[\hskip-5pt\begin{array}{*{16}{c}}#2\end{array}\hskip-5pt\right]}

%% file: sec/abstract.tex
\begin{abstract}
    \noindent Multiple rotation averaging (MRA) is a fundamental optimization problem in 3D vision and robotics that aims to recover globally consistent absolute rotations from noisy relative measurements. 
    Established classical methods, such as L1-IRLS and Shonan, face limitations including local minima susceptibility and reliance on convex relaxations that fail to preserve the exact manifold geometry, leading to reduced accuracy in high-noise scenarios.
    We introduce IQARS~(\textbf{I}terative \textbf{Q}uantum \textbf{A}nnealing for \textbf{R}otation \textbf{S}ynchronization), the first algorithm that reformulates MRA as a sequence of local quadratic non-convex sub-problems executable on quantum annealers after binarization, to leverage inherent hardware advantages.
    IQARS removes convex relaxation dependence and better preserves non-Euclidean rotation manifold geometry while leveraging quantum tunneling and parallelism for efficient solution space exploration.
    We evaluate IQARS's performance on synthetic and real-world datasets.
    While current annealers remain in their nascent phase and only support solving problems of limited scale with constrained performance, we observed that IQARS on D-Wave annealers can already achieve $\approx$12\% higher accuracy than Shonan---the best-performing classical method evaluated empirically.
    Project page: \url{https://4dqv.mpi-inf.mpg.de/QMRA/}.
\end{abstract}

%% file: sec/paper.tex
\section{Introduction}
\noindent Multiple rotation averaging~(MRA) stands as a fundamental group synchronization problem in 3D computer vision, aiming to recover globally consistent absolute rotations from noisy relative measurements while satisfying compositional constraints
on the \(\text{SO}(3)\) manifold~\cite{ozyecsil2017survey,hartley2013rotation,chatterjee2017robust,chatterjee2013efficient,sidhartha2021all}.
Applications and research fields that can benefit from MRA include structure-from-motion~(SfM)~\cite{cui2017hsfm,cui2015global,moulon2013global,lindenberger2021pixel}, simultaneous location and mapping (SLAM)~\cite{macario2022comprehensive,kazerouni2022survey,matsuki2024gaussian,zhu2022nice},  multi-view reconstruction \cite{goesele2006multi,yao2020blendedmvs}, virtual reality \cite{wohlgenannt2020virtual,hurst2011mobile,nehme2020comparison} and robotics \cite{maier2012real,jiang2022autonomous}, among others. 
A common approach to averaging multiple rotations involves enforcing cycle consistency and distributing errors across all cameras to ensure that the composition of rotations between any two cameras closely aligns with their respective partial measurements, i.e.,~$
R_{ij} = {R}_j {R}_i^\top$ for all $i \text{ and } j$. 

MRA presents several unique challenges such as: (1) inherent non-convexity of \(\text{SO}(3)\)'s Lie group structure, whose Riemannian geometry creates complex optimization landscapes; (2) the inevitable presence of noise in pairwise measurements, arising from feature matching errors, which disrupt cycle consistency and introduce pathological local minima~\cite{wilson2016rotations}.
These challenges manifest mathematically as an increased susceptibility of gradient-based methods to convergence in sub-optimal local minima.
Classical solvers minimize geodesic errors using iterative reweighted least squares~\cite{chatterjee2013efficient, chatterjee2017robust} like L1-IRLS, a dominating robust approach that efficiently rejects outliers \textit{but can easily get stuck in local minimas.}
More recent Shonan method provides certifiable optimality through semidefinite convex relaxation at $O(N^3)$ computational cost \textit{but it suffers from increasing relaxation gap in presence of noise}; $N$ is the number of cameras~\cite{dellaert2020shonan, wang2013exact, fredriksson2012simultaneous}.
Developing accurate methods for MRA, therefore, remains an open challenge. 

\begin{figure}[t] 
    \centering 
    \includegraphics[width = 0.48\textwidth]{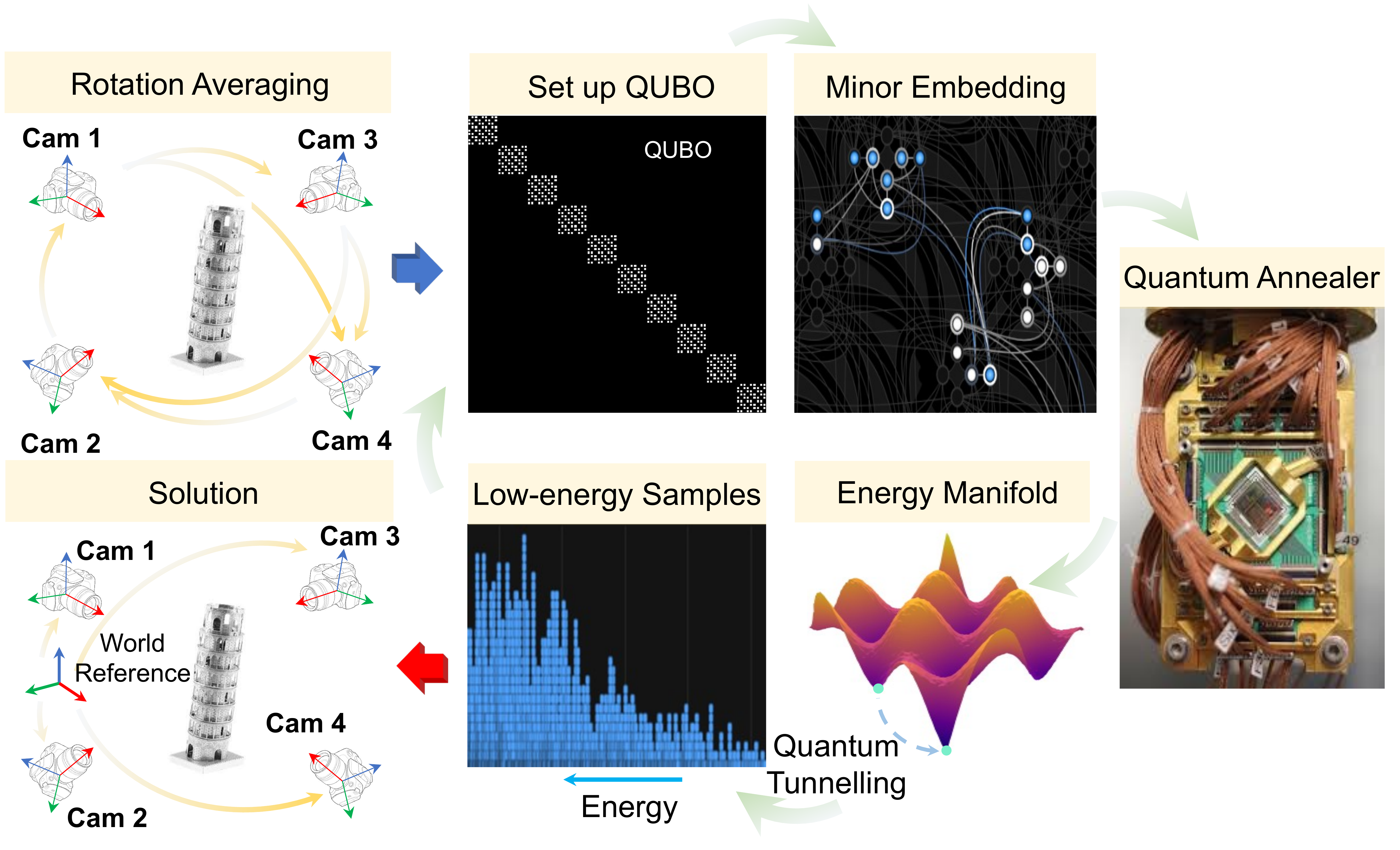} 
    \caption{IQARS formulates MRA as QUBO sub-problems executable via quantum annealing to efficiently search for high-quality solutions.
    Estimated rotations are acquired from QUBO solutions once converged. 
    Annealer's image is taken from \cite{hsu2013dwave}.} 
    \vspace{-13pt}
    \label{Preview}
\end{figure}

Quantum annealers are specialized quantum computing devices designed to sample high-quality solutions to NP-hard quadratic unconstrained binary optimization~(QUBO) problems.
Leveraging quantum mechanical properties such as quantum parallelism and tunneling, such devices enable efficient exploration of complex energy landscapes compared to classical methods, and facilitate exploration of rugged energy landscapes with multiple low-energy basins.
%
Empirical studies have substantiated their computational capabilities, demonstrating practical advantages of quantum effects in challenging optimization tasks~\cite{Lanting2014,denchev2016computational,King2024arXiv,KueteMeli2025arXiv}. 
These advantages have further motivated exploring applications of quantum annealing to practical applications; see Sec.~\ref{sec:related_work}. 

Motivated by recent advances in quantum-enhanced optimization, we identify that MRA can potentially benefit from quantum annealing to circumvent key limitations in classical approaches.
In this work, we propose IQARS---a formulation of MRA designed for execution on Ising machines such as quantum annealers, which can globally synchronize multiple rotations with high accuracy; see Fig.~\ref{Preview} for an overview. 
The practical annealing process is often perturbed by hardware-induced noise, such as thermal fluctuations, and yields sub-optimal sample solutions in the vicinity of the global optima.
We hence propose a posterior framework to refine the solution quality that exploits the observation that, under moderate noise conditions maintaining quasi-equilibrium dynamics, the annealer's output distribution approximates a Boltzmann distribution~\cite{zaech2024probabilistic}.
This framework effectively compensates for imperfections while preserving the sampling properties of QAs.
In summary, the technical contributions of this work include: 
\begin{itemize} 
  \item[$\bullet$] IQARS, the first end-to-end framework for solving MRA on Ising machines such as quantum annealers;
  \item[$\bullet$] Sequential QUBO formulations that strictly adhere to $\text{SO}(3)$ manifold geometry during optimization with monotonic residual decrease;
\end{itemize} 

\noindent While quantum annealers remain in their nascent phase with limited problem-scale capacity, empirical evaluation of IQARS on D-Wave annealing hardware shows promising performance: ${\approx}12\%$ in average lower residual than Shonan—the best-performing classical method evaluated empirically—under noisy measurements.
In the MRA setting, our posterior analysis statistically recovers higher-fidelity solutions with 60\% probability.
Collectively, these advancements position IQARS as the first MRA method leveraging quantum annealing, with distinct characteristics arising from annealing-based optimization.

\section{Related Work} 
\label{sec:related_work} 

\noindent \textbf{Rotation Averaging}. MRA became the de facto approach in contemporary 3D reconstruction pipelines for synchronizing locally estimated, noisy relative rotations into globally consistent absolute rotations~\cite{cui2017hsfm,cui2015global,lee2020robust,gao2021incremental,chen2021hybrid,chatterjee2017robust,chatterjee2013efficient,eriksson2018rotation}. 
Chatterjee et al.~\cite{chatterjee2013efficient} introduced a pioneering approach: L1-IRLS, an iterative approach for MRA with dynamic residual weighting for robust outlier rejection.
Eriksson et al.~\cite{eriksson2018rotation} proposed a rotation averaging solution through Lagrangian duality and semidefinite relaxation, while Dellaert et al.~\cite{dellaert2020shonan} formulated Shonan's dimensional lifting approach, solving the problem in $\text{SO}(p)$ ($p>3$) space before projecting to $\text{SO}(3)$.
This guarantees global convergence under mild noise while maintaining geometric constraints, i.e. orthogonality and unit-determinant, inherent in rotation averaging.
Gao et al.~\cite{gao2021incremental} presented an incremental estimation workflow for rotation averaging, circumventing the cubic-time complexity of simultaneous absolute rotation estimation.
Lee et al.~\cite{lee2020robust} proposed a robust Weiszfeld-based rotation averaging method, enhancing both computational efficiency and convergence reliability.
This, however, only works for single-rotation estimation.

\noindent \textbf{Annealing-based Quantum-enhanced Computer Vision~(QeCV)}. Quantum annealing has emerged as a promising approach to address some challenging problems in computer vision \cite{KueteMeli2025arXiv}. 
Pioneering works by Golyanik et al.~\cite{golyanik2020quantum} demonstrate its application to 2D/3D point set rotation estimation using binary-weighted basis matrices within a gravitational correspondence framework.
Birdal et al.~\cite{birdal2021quantum} leveraged quantum annealers for permutation synchronization problems, solving for cycle-consistent discrete permutations. 
Farina et al.~\cite{farina2023quantum} introduced a quantum-enhanced feature matching via preference-consensus matrices under known model count constraints.
Annealing-based approaches have since been adapted to and benefit diverse vision tasks, including shape matching~\cite{bhatia2023ccuantumm, benkner2021q}, motion segmentation~\cite{arrigoni2022quantum}, object detection~\cite{li2020quantum}, stereo matching~\cite{braunstein2024quantum}, and super-resolution~\cite{choong2023quantum}, collectively demonstrating quantum annealing's potential, though challenges remain in scaling these methods to practical problem sizes.

Building upon prior work in transformation (mainly rotation) estimation~\cite{golyanik2020quantum, meli2022iterative}, which emphasizes estimating pairwise rotations between point sets, we propose a quantum-annealing-based synchronization paradigm for global multiple rotations.
Our IQARS framework inherently preserves the $\text{SO}(3)$ manifold structure through the optimization process.
%
This fundamental methodological advancement leads to significantly increased performance, as will be demonstrated empirically.

\section{Review: Adiabatic Quantum Computing for Solving QUBO Problems} \label{s:review_QUBO}

\noindent Adiabatic quantum computing (AQC) formulates combinatorial optimization problems as Hamiltonian ground-state searches and seeks to prepare low-energy states through slow adiabatic evolution; see App.~\ref{AQC} for details.
The protocol involves an adiabatic transition between an initial Hamiltonian $H_0$, with a computationally tractable ground state (e.g., $H_0 = -\sum_{i=1}^n X_i$, where $X_i$ is the Pauli-$X$ operator acting on the $i$-th qubit of an $n$-qubit quantum system), and the problem Hamiltonian $H_P$, which encodes the solution to the target optimization problem. The time-dependent Hamiltonian is given by:
\begin{equation}\label{eq:time_dependent_hamiltonian}
H(t) = \left(1 - s(t)\right)H_0 + s(t)H_P, \quad s(t) \in [0, 1],
\end{equation}
where $s(t)$ is a scheduling function that governs the adiabatic evolution, smoothly transitioning the system from $H_0$ (at $s(t) = 0$) to $H_P$ (at $s(t) = 1$).

\noindent \textbf{QUBO as an AQC-Compatible Problem.}
QUBO problems, defined as:
\begin{equation}\label{eq:qubo}
\min_{\mathbf{x} \in \{0, 1\}^n} \mathbf{x}^\top Q \mathbf{x} + \mathbf{c}^\top \mathbf{x},
\end{equation}
where $Q \in \mathbb{R}^{n \times n}$ is a symmetric matrix and $\mathbf{c} \in \mathbb{R}^n$, are naturally suited for AQC due to their quadratic structure, which aligns with the interaction terms in quantum Hamiltonians.
To embed a QUBO problem into AQC, each binary variable $x_i \in \{0, 1\}$ is mapped to a qubit, where the classical states $0$ and $1$ correspond to the eigenstates of the Pauli-$Z$ operator $Z_i$, with eigenvalues $+1$ and $-1$, respectively. This mapping is formalized via the transformation: $x_i \mapsto \frac{1}{2}\left(I - Z_i\right) $. Substituting this into the QUBO objective function $f({x})$ yields a quantum Hamiltonian:
\begin{equation}\label{eq:full_hamiltonian}
H_P = \sum_{i=1}^n \sum_{j \geq i}^n Q_{ij} \frac{1}{4}\left(I - Z_i\right)\left(I - Z_j\right) + \sum_{i=1}^n c_i \frac{1}{2}\left(I - Z_i\right).
\end{equation}
Upon simplification, constant terms are eliminated, resulting in an Ising-type Hamiltonian:
\begin{equation}\label{eq:ising_hamiltonian}
H_P = \sum_{i=1}^n h_i Z_i + \sum_{i<j}^n J_{ij} Z_i Z_j,
\end{equation}
where $h_i$ represents local magnetic fields and $J_{ij}$ encodes inter-qubit-qubit couplings. The ground state of $H_P$ encodes the optimal solution to the original QUBO problem \eqref{eq:qubo}, thereby establishing a connection between combinatorial optimization and quantum adiabatic evolution.

\section{Quantum MRA} \label{sec:quantumRA}

\noindent The classical MRA problem objective is:
\begin{equation} \label{RA}
    \min_{ R_1,  \dots,  R_N \in \text{SO}(3)} \sum_{(i,j)} \| \tilde{ R}_{ij}  R_i -  R_j \|_F^2.
\end{equation}
$\text{SO} (3) $ is the set of matrix elements $R$ that are orthogonal and have a unit determinant. $\left\{ R_1, \cdots,  R_N\right\} \subset \text{SO} (3)$ represent a set of absolute $3$D rotational matrices w.r.t.~a global reference frame. $\tilde{R}_{ij}$ is the observed relative rotation; see App.~\ref{sec:preliminaries} for detailed review.
The process of transforming Eq.~\eqref{RA} into a QUBO form that can be sampled on a quantum annealer can be decomposed into two independent primary steps: 1) reformulating the original optimization problem~\eqref{RA_final} into an equivalent quadratic optimization problem through algebraic manipulation and constraint embedding; 2) representing rotation-parameterizing variables as binary strings (e.g., via fixed-point quantization).
We present intermediate steps below:
\begin{proposition} \label{proposition1}
Problem \eqref{RA_final} can be formulated as a quadratic optimization problem in matrix form:
\begin{equation}
\label{eq:matricization}
   \min_{ R_1, \dots,  R_N}  \sum_{(i,j)} \| \tilde R_{ij}  R_i -  R_j \|_F^2 
   \ = \ \min_{ R} \   R^\top {Q}  R,
\end{equation}
with
$  R =
        \Matrix{
        \cdots \
        \vect( R_i)^\top \
        \cdots
        }^\top
        \in \mathbb{R}^{9N\times 1}$, 
where 
``$\vect(\cdot)$'' denotes the vectorization of a matrix by stacking its columns, and ${Q} \in \mathbb{R}^{ 9N \times 9N}$ is the cost matrix:
\begin{equation} \label{Q_cal}
{Q} = 
    -2 
    \Matrix{
         I \otimes  R_{11} ^\top & \cdots &  I \otimes  R_{1N} ^\top \\
        \vdots & \ddots & \vdots \\  I \otimes  R_{N1} ^\top & \cdots &  I \otimes  R_{NN} ^\top
    }.
\end{equation}
\end{proposition}

\begin{proof}
    see App.~\ref{proof_p1}.
\end{proof}

\subsection{Taylor Linearization on {\large $\text{SO}(3)$} Manifold} \label{SO(3)}

\noindent Instead of parametrizing rotation matrices on a binary-weighted matrix basis~\cite{golyanik2020quantum}, we leverage the relationship between the orthogonal Lie group $\text{SO}(3)$ and its algebra $\mathfrak{so}(3)$ to ensure optimization strictly within the $\text{SO}(3)$ manifold~\cite{meli2022iterative}. 
Any matrix ${R}_i\in\text{SO}(3)$ can be generated from its tangent vector $\bm{v}_i = (\vit_i^1, \vit_i^2, \vit_i^3)^\top \in \mathbb{R}^3$ via the exponential map implemented through Rodrigues' rotation formula:
\begin{equation} \label{eq:exponentialmap}
{R}_i (\bm{v}_i) = \exp(\mathcal{R}(\bm{v}_i)) = I + \frac{\sin\theta}{\theta} \mathcal{R}(\bm{v}_i) + \frac{1-\cos\theta}{\theta^2} \mathcal{R}^2(\bm{v}_i),
\end{equation}
where $\mathcal{R}: \mathbb{R}^3 \to \mathfrak{so}(3)$ denotes Lie algebra isomorphism:
\begin{equation}
\mathcal{R}(\bm{v}_i) := \begin{pmatrix}
0 & -\vit_i^3 & \vit_i^2 \\
\vit_i^3 & 0 & -\vit_i^1 \\
-\vit_i^2 & \vit_i^1 & 0
\end{pmatrix}, \quad \mathcal{R}(\bm{v}_i)^\top = -\mathcal{R}(\bm{v}_i).
\end{equation}
The magnitude $\theta = \|\bm{v}_i\|_2 \in \mathbb{R}$ defines the rotation angle and the unit vector $\bm{x}_i = \bm{v}_i/\|\bm{v}_i\|_2$ defines the rotation axis. 
By aggregating all tangent vectors $\bm{v}_i$ into $\bm{v} = (\ldots, \bm{v}_i^\top, \ldots)^\top \in \mathbb{R}^{3N\times1}$ and vectorizing the rotations as $R(\bm{v}) = (\ldots, \text{vec}(R_i(\bm{v}_i))^\top, \ldots)^\top \in \mathbb{R}^{9N \times 1}$, we reformulate the right-hand side of Eq.~\eqref{eq:matricization} as:
\begin{equation} \label{eq:unconstrained}
\min_{\bm{v}\in\mathbb{R}^{3N}} R(\bm{v})^\top Q R(\bm{v}).
\end{equation}
This transformation reduces the search space from a nonlinear $\text{SO}(3)$ manifold to Euclidean tangent vectors $\bm{v}_i$, where we solve for $\argmin$ of problem \eqref{eq:unconstrained}.
Due to the non-linearity of the exponential map, we expand $ R_i$ around its current guess $\bm{v}_i ^{k}$ up to the first order in the tangent space.
The next iterate $\bm{v}_i ^{k+1}$ is obtained via solving the resulting locally non-convex quadratic sub-problem, ensuring that updates always remain within the tangent space of \(\text{SO}(3)\).

Specifically, we iteratively replace $ R(\bm v)$ in problem \eqref{eq:unconstrained} with its first-order Taylor linearization:
\begin{equation} \label{eqn:linearization}
\Hat{R}^k (\Delta \bm{v}) := R(\bm{v} ^{k}) + \nabla R (\bm{v} ^{k})^\top \Delta \bm{v},
\end{equation} 
and optimize over $\Delta \bm{v}$, i.e., 
\begin{equation}
\label{eq:linearized}
\begin{aligned}
   \Delta \bm{v}^{k+1} &= \argmin_{\Delta \bm v \in \mathbb{R}^{3N}} \  \Hat{R}^k (\Delta \bm{v}) ^\top {Q} \Hat{R}^k (\Delta \bm{v}), \\
   \bm{v}^{k+1}  &= \bm{v}^{k} +  \Delta \bm{v}^{k+1},
   \end{aligned}
\end{equation}

\noindent In Eq.~\eqref{eqn:linearization}, the gradient $\nabla  R(\bm{v}^{k}) \in \R^{3N \times 9N}$ is a block diagonal matrix with partial derivative-wise vectorized $\nabla  R_i(\bm v_i^{k})\in \R^{3 \times 9}$ over the diagonal.
\begin{remark}
    While the exact mapping $R(\bm{v}^{k+1})$ inherently preserves the $\text{SO}(3)$ manifold structure by construction, its first-order Taylor approximation $\Hat{R}^k(\Delta \bm{v}^{k+1})$ can deviate from $\text{SO}(3)$ on the scale of $\mathcal{O}(\|\Delta \bm{v}^{k+1}\|^2)$.
\end{remark}
To improve approximation fidelity while ensuring convergence, we propose: 1) iteration-adaptive box constraint in \eqref{eq:linearized} (similar to the one used in Ref.~\cite{delilbasic2023single,meli2022iterative}), i.e.~enforce  
\begin{equation}
\label{eq:centering}
    \|\Delta \bm{v}\|_\infty \leq \delta^k, 
\end{equation} 
and 2) an orthogonality-promoting soft penalty term $\alpha\|\Hat{R}^k(\Delta \bm{v})\|_F^2$ (see App. \ref{penalization_effect} on how the penalty encourages the (blocks of the) Taylor approximation to maintain $\text{SO} (3)$ proximity).
We summarize our approach as follows:
\begin{proposition} \label{proporsition2}
Inserting the linearization \eqref{eqn:linearization} into the minimization Problem \eqref{eq:linearized} and incorporating the softly-regularized box constraint, we obtain the QUBO problem: 
\begin{equation}
\label{eq:final}
\begin{aligned}
   \Delta \bm{v}^{k+1} &= \argmin_{\|\Delta \bm v\|_\infty \leq \delta^k} \  \Delta \bm{v} ^\top \hat{{Q}} \Delta \bm{v} +   \hat{c}^\top \Delta \bm{v} , \\
   \bm{v}^{k+1}  &= \bm{v}^{k} +  \Delta \bm{v}^{k+1},
   \end{aligned}
\end{equation}
with
\begin{equation} \label{linearize}
    \begin{aligned}
        \hat{ Q} &:= \nabla  R(\bm{v}^{k}) \, ( Q + \alpha N  I_{9N}) \, \nabla  R(\bm{v}^{k})^\top, \\
        \hat{\bm c}  &:= 2 \, \nabla  R(\bm{v}^{k}) \, ( Q + \alpha N  I_{9N}) \,  R(\bm{v}^{k}).
    \end{aligned}
\end{equation}
\end{proposition}

\begin{figure}[t] 
    \centering 
    \includegraphics[width = 0.48\textwidth]{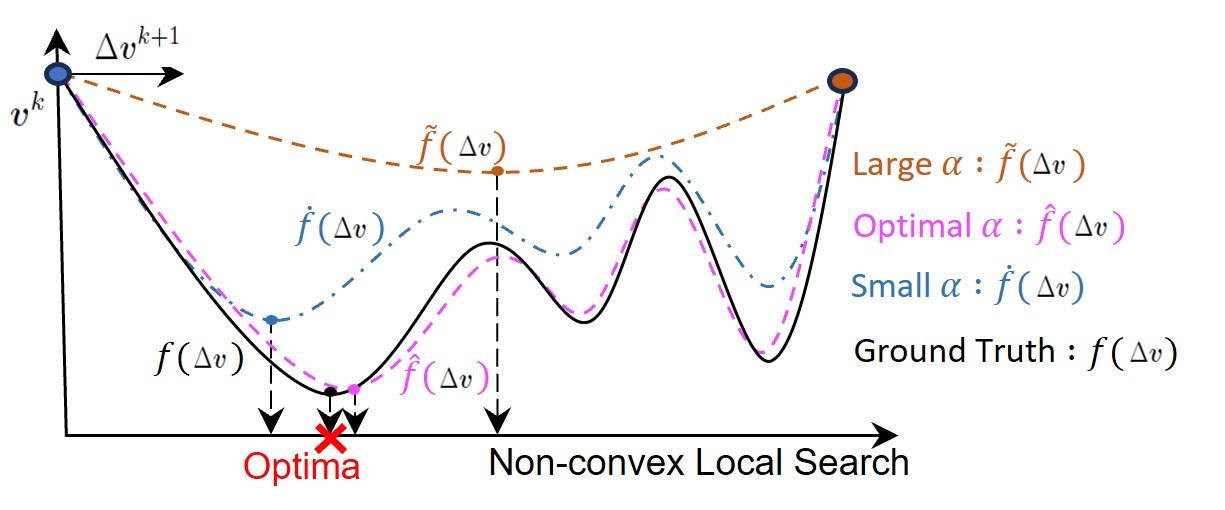} 
    \caption{Schematic local search behavior visualization under varying penalty strengths $\alpha$.} 
    \vspace{-5pt}
    \label{optimim_alpha}
\end{figure}
We observe that aggressive regularization can enforce positive semi-definiteness~(PSD) in $\hat{Q}$—--thereby convexifying Eq.~\eqref{eq:final} and enabling classical efficient solvability.
However, it simultaneously distorts the landscape by overwhelming the objective function and introduces artificial minima that fundamentally misrepresent the true optimization landscape.
This marks calibration of $\alpha$ crucial for the optimal performance such that: 1) Taylor approximation's deviation from $\text{SO}(3)$ within $\mathcal{O}(\|\Delta\bm{v}\|^2)$ bounds is constrained and 2) distortion of the original optimization landscape is minimized; see Fig.~\ref{optimim_alpha} for a schematic visualization.
\begin{remark}
    The exponential map $\mathfrak{so}(3) \to \text{SO}(3)$ preserves Lie group traversal at each iteration $\bm{v}^{k+1} = \bm{v}^k + \Delta\bm{v}$ throughout the optimization.
\end{remark}

\subsection{Binary Encoding of Iterative Update} \label{binary_encoding}

\noindent We discretize the search space $\Delta \bm{v}$ and leverage binary encodings to represent the discretized space for compatibility with the qubit-based Ising model supported by the annealer; see Sec.~\ref{s:review_QUBO}.
Specifically, we represent $\Delta \bm{v}$ via a binary vector $\bm{q} \in \{0, 1\}^m$, where each binary variable $q_{i}$ corresponds to the state of a logical qubit on the annealer. 
For hardware implementations (e.g., D-Wave systems), these binary variables are often mapped to spin states $\sigma_{i} \in \{-1, +1\}$ via the transformation $q_{i} = \frac{\sigma_{i} + 1}{2}$; see Sec.~\ref{s:review_QUBO}.
To satisfy the constraint $\|\Delta \bm{v}\|_\infty \leq \delta^k$, we discretize the interval $[-\delta^k, \delta^k]$ uniformly using $m$ qubits per dimension.
Specifically, we express $i$-th discretized value of $\Delta \bm{v}$ as:
\begin{equation} \label{interval_rep}
    (\Delta \bm{v})_i = -\delta^k + \frac{2 \delta^{k}}{s} \sum_{\ell = 0}^{m-1} 2^{\ell} q_{i,\ell}, \quad \text{where} \quad s = 2^m - 1.
\end{equation}
where factor $2 \delta^{k}/s$ with $s := 2^m - 1$ serves as a normalization factor and ensures that  $\|\Delta \bm v\|_\infty \leq \delta^k$.
Note that Eq.~\eqref{interval_rep} discretizes the search space to 
\begin{equation} \label{discretized_search_space}
(\Delta \bm v)_i =
\left\{
-\delta^k + \frac{2 T \delta^{k}}{s}
\;:\;
T \in \mathbb{N},\;
T \leq 2^m - 1
\right\}
\end{equation}
representing \textit{exponentially-many, i.e. $2^m$}, equidistant points per dimension.
We then write the variable vector $\Delta \bm v$ more compactly via a matrix-vector product as 
\begin{equation}
\label{eq:trick2}
    \Delta \bm v = -\delta^k\mathbf{1}_{3N} +  D \bm{q},
\end{equation}
where
\begin{equation} \label{binarization}
    \begin{aligned} 
        \bm{q} &= 
        \Matrix{
        \bm{q}_1 \\
        \vdots \\
        \bm{q}_{3N}
        } \in \mathbb{R}^{3Nm}
        \quad ,\text{with} \quad
        \bm{q}_i = 
        \Matrix{
        q_{i, 0} \\
        \vdots \\
        q_{i, m-1}
        } \in \mathbb{R}^{m},
        \\
         D &= 
        \frac{2}{s} 
        \Matrix{
        \delta^{k} &  & {0} \\
        & \ddots &  \\
        {0} &  & \delta^{k}
        }
        \otimes \Matrix{2^0 & \cdots &2^{m-1}} \in \mathbb{R}^{ 3N \times 3Nm},
    \end{aligned}
\end{equation}
$\mathbf{1}_{3N}$ is a vector of length $3N$ with all entries being one, and $\bm{q} \in \{0, 1\}^{3Nm}$ stacks all binary variables.
%
\begin{remark}
The binary encoding's approximation precision when representing arbitrary floating-point values improves exponentially with scaling $\mathcal{O}(2^m)$.
\end{remark}

\begin{proposition} \label{proporsition3}
Binary encoding of the continuous search interval makes $\bm q \in \B^{3Nm}$ the new optimization variable, changing Problem~\eqref{eq:final} into
    \begin{equation} \label{final_QUBO}
        \argmin_{\bm{q} \in \B^{3Nm}} \  \bm{q}^\top \tilde{ Q} \bm{q} + \bm{q}^\top \tilde{\bm c},
    \end{equation}
    with
\begin{equation} \label{linearize2}
    \begin{aligned}
        \tilde{ Q} &:=  D^\top \hat{ Q}  D, \\
        \tilde{\bm c}  &:=   D^\top (\hat{\bm c} - 2\delta^k\hat{ Q}\mathbf{1}_{3N}). 
    \end{aligned}
\end{equation}
\end{proposition}
\begin{proof}
    See App.~\ref{proof_p3}.
\end{proof}
Upon obtaining the optimal $\bm{q}$ via \cref{final_QUBO}, we update the solution iterate as $\bm v^{k+1}$ via \cref{eq:trick2,eq:linearized}.
The iteration continues until the energy update converges.

\subsection{Required Hardware Resource Analysis} \label{scalability}
\noindent Practical implementation of IQARS and its performance is constrained by available hardware resources.
For synchronizing a camera graph of $N$ nodes, our algorithm maintains favorable resource scaling properties and require $3Nm$ logical qubits (see Prop.~\ref{proporsition3}) while achieving floating-point representation error on the order of $\mathcal{O}(\delta^{(k)}/2^m)$. 
The scaling behavior demonstrates: 1) \textit{linear} dependence on problem size and 2) \textit{logarithmic} dependence on solution precision, suggesting strong algorithmic potentials for large-scale MRA problems as hardware continues to evolve.
\begin{remark}
    With limited quantum resource budgets, the multiplicative nature of resource requirements introduces trade-offs between problem scales and solution quality.
\end{remark}

\subsection{Algorithmic Protocol of IQARS}

Large-scale MRA problems present complex optimization landscapes that typically require fine discretization, i.e. large $m$, for accurate solutions.
However, current quantum hardware limitations impose strict constraints on available qubit resources for such a purpose.
If discretization is coarse, the global optimum of Eq.~\eqref{final_QUBO} can easily lie outside the discretized search grid defined by Eq.~\eqref{discretized_search_space}, potentially leading to oscillatory convergence behavior.
To ensure robust convergence behavior, we implement an adaptive search protocol inspired by \cite{berger1984adaptive,berger1989local,meli2022iterative} that dynamically adjusts the search radius $\delta$ based on local update magnitudes during optimization.
While this heuristic approach does not provide formal convergence guarantees---particularly with respect to maintaining global optimality through successive radius contractions---our empirical results demonstrate its effectiveness in achieving stable convergence to high-quality solutions.
This adaptive mechanism proves particularly valuable in resource-constrained scenarios where coarse discretization would otherwise severely limit solution quality.
We outline the algorithm in Alg.~\ref{IterativeQRA}.
\begin{algorithm}[H] 
\caption{\textit{Iterative Quantum Approach for Rotation Synchronization (IQARS)}} \label{IterativeQRA}
\begin{algorithmic}[1]
    \State \textbf{Input:} $ R_{ij}$: Pairwise relative rotations 
    \Statex \hspace{3em} Maxiter: Maximum number of iterations
    \Statex \hspace{3em} $\bm{v}^{0}$: Initial rotation guess
    \Statex \hspace{3em} $\delta^{0}$: Search window radius 
    \Statex \hspace{3em} $\kappa$ : Threshold for decreasing search radius $\delta$
    \Statex \hspace{3em} $m$: Qubits required for search discretization
    \Statex \hspace{3em} $\epsilon$: Stopping criterion
    \Statex \hspace{3em} $\tau$: Annealing rate
    \State \textbf{Initialization:} $j$, $e$ $\leftarrow$ 0;\quad Construct $ Q$ according to Eq.~\eqref{Q_cal};\quad Initialized identity rotations: $ R(\bm v^{0}) = I$.
    \State \textbf{while} $j$ $\leq$ Maxiter \textbf{do}:
    \State \hspace{1em} Compute gradient $\nabla  R(\bm{v}^{j})$ according to Eq.~\eqref{eq:exponentialmap}
    \State \hspace{1em} Construct $\hat{ Q}$ and $\hat{\bm c}$ according to Eq.~\eqref{linearize}
    \State \hspace{1em} Construct $ D$ according to Eq.~\eqref{binarization}
    \State \hspace{1em} Solve the QUBO problem~\eqref{final_QUBO} and measure binary
    \Statex \hspace{1em} solution vector $\mathbf{q}$
    \State \hspace{1em} $j \leftarrow j + 1$
    \State \hspace{1em} Update solution $\bm{v}^{j}$ and $ R(\bm{v}^{j})$ using Eq.~\eqref{eq:linearized} and 
    \State \hspace{1em} \eqref{eq:exponentialmap};\space update $e \leftarrow \operatorname{Avg}\,\sum_{(i,j)} \| R_{ij}  R_i -  R_j \|_F^2$
    \State \hspace{1em} \textbf{If} $e < \epsilon$ \textbf{do}: \space $\operatorname{break}$
    \State \hspace{1em} \textbf{Else}:
    \State \hspace{3em} \textbf{If} $|| R(\bm{v}^{j}) -  R(\bm{v}^{j-1})||_F < \kappa$  \textbf{do}:
    \State \hspace{5em} $\delta \leftarrow \frac{\delta}{\tau}$; \space $\kappa \leftarrow \frac{\kappa}{\tau}$ 
    \State \hspace{3em} \textbf{End}
\State \textbf{End} 
\State \Return Residual: $e$; iteration count: $j$; synchronized rotations: $\bm{v}^{j},  R(\bm{v}^{j})$. 
\end{algorithmic}
\end{algorithm}
\vspace{-10pt}

\subsection{Solution Refinement via Posterior Analysis} \label{sec:postanal}

\noindent Quantum annealers return solutions by sampling from low-energy states of a pre-defined QUBO energy landscape.
Under ideal adiabatic conditions, the system will converge to the ground state, returning the global optimal solution.
However, ensuring adiabatic evolution requires problem-specific annealing schedules, and in practice, environmental interactions such as thermal excitations could induce non-adiabatic effects that lead to a distribution of low-energy states well-approximated by classical Boltzmann distributions rather than deterministic ground-state~\cite{zaech2024probabilistic,schuld2021machine}.
To improve the solution quality due to non-adiabatic effects leveraging the probabilistic outputs, we propose a \textit{physics-inspired energy-weighted voting} approach that refines solution quality.
\begin{figure}[t]
    \centering 
    \includegraphics[width = 0.48\textwidth]{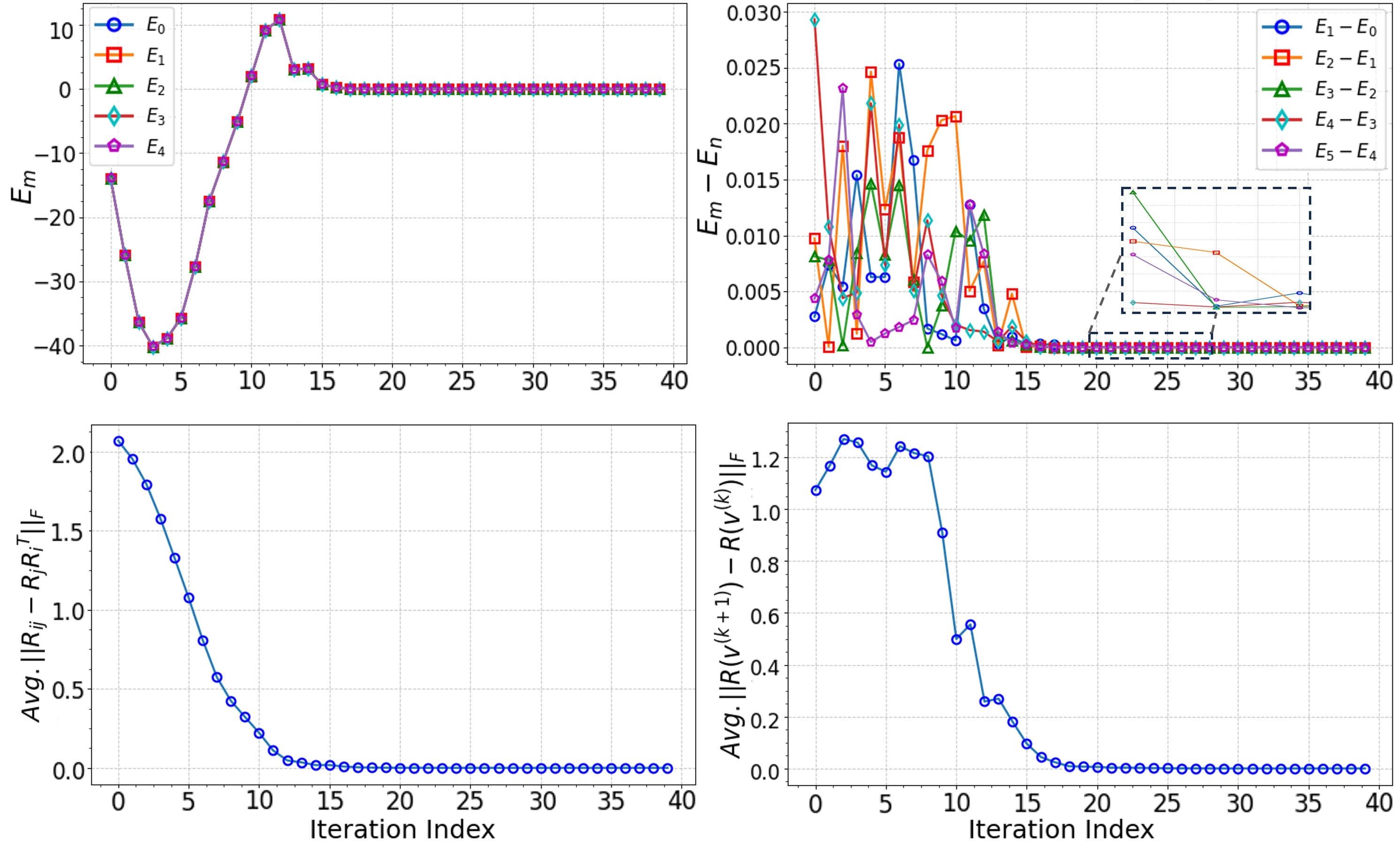} 
    \caption{IQARS convergence behavior for synchronizing $N{=}10$ synthetic random noiseless rotations.} 
    \vspace{-5pt}
    \label{convergence}
\end{figure}
The method assigns Boltzmann weights $w_i$:
\begin{equation} \label{Boltzmann}
    w_i = \frac{e^{-\beta E(\mathbf{q}_i)}}{\mathcal{Z}}, \quad \mathcal{Z} = \sum_{j=1}^K e^{-\beta E(\mathbf{q}_j)},
\end{equation}
to each of the $K$-lowest energy solutions $\{\mathbf{q}_i\}_{i=1}^K$, where $\mathcal{Z} = \sum_{j=1}^K e^{-\beta E(\mathbf{q}_j)}$ serves as the partition function and $E(\mathbf{q}_i)$ denotes the solution energy.
$\beta$ represents an effective inverse temperature parameter.
This creates a noise-robust aggregation scheme that naturally prioritizes lower-energy configurations.
%
Details of the posterior refinement protocol are formally described in App.~\ref{posterior_protocol}. 
\textit{While this protocol is evaluated specifically for MRA, we emphasize that it establishes a general methodology for enhancing solutions obtained from quantum annealing. The approach's theoretical foundations in statistical mechanics and its demonstrated efficacy suggest significant potential for application to a range of other problems, which remain to be explored.}

\section{Experiments} 
\label{sec:experiments}

\noindent We empirically evaluate IQARS's accuracy and robustness using synthetic and real-world datasets~\cite{strecha2008benchmarking}.
The QUBO parametrization---comprising coupling matrices $\tilde{Q}$ and linear bias term $\bm{\tilde{c}}$---is prepared on an Intel Xeon CPU @ 2.20 GHz with 12 GB RAM.
Prepared QUBO formulations follow by execution on D-Wave \textit{Advantage Series} 6.4 annealers operating at a cryogenic base temperature of \( 15 \) \textit{mK}.\
All experimental configurations remain consistent unless explicitly mentioned otherwise.
The annealing protocol maintains the default duration of 20 \textit{µs} apart from data transmission overheads.
Each experimental trial is conducted with $10^2$ reads per iteration.
The search space is configured with an initial value of $\delta^{0} = \pi/30$  discretized uniformly with $m=3$ qubits, yielding $2^3=8$ equidistant sub-intervals---a coarse discretization scheme explicitly chosen to accommodate a larger problem scale $N$ with constrained quantum computational budget; see details in Sec.~\ref{scalability}.

\noindent \textbf{Noise-free Synthetic Rotation Graphs}. We construct a fully-connected synthetic rotation graph of $N$ nodes.
Each node \( i \) is initialized with a random tangent vector \( \bm{v_i} \in \mathbb{R}^3 \) uniformly sampled from the Lie algebra \(\mathfrak{so}(3)\) and later mapped to its corresponding rotation matrix \( R(\bm{v_i}) \) via the exponential mapping defined in Eq.~\eqref{eq:exponentialmap}.
For each node pair \( \{i, j\} \), we compute the relative rotation, constructing a dense symmetric matrix \( \tilde{R}_{ij} \in R^{N \times N} \) that encapsulates the relative rotations within the camera graph. 
Leveraging the constructed dense relative rotation matrix, we implement Alg.~\ref{IterativeQRA} to recover synchronized absolute rotations \( R_i \) while systematically monitoring the convergence dynamics through quantitative metrics including evolution of energy spectrum $E_m$, neighboring energy gaps $E_m - E_n$, and rotation residuals throughout iterations, as visualized in Fig.~\ref{convergence}.
The observed monotonic decay of rotation residuals along with the absence of energy gap closures provides empirical validation for the theoretical convergence guarantees established in Theorem~\ref{thm:adiabatic}.
In addition to the visualized convergence dynamics, the performance of synchronizing noise-free synthetic rotations of different scales with IQARS is recorded in Tab.~\ref{noiseless}. 
Besides IQARS---which strictly maintains the optimization path within the SO(3) manifold through iterative exponential mapping; see Sec.~\ref{SO(3)}---we also evaluate an alternative direct approach proposed in~\cite{golyanik2020quantum} that can be combined with Prop.~\ref{proposition1} for MRA.
This approach, however, lacks formal SO(3) constraint guarantees; see App.~\ref{ours:direct} for implementation details.
We denote IQARS as ``ours (iterative)'', and the alternative approach as ``ours (direct)'', with the experimental results summarized in Tab.~\ref{noiseless} to enable a principled comparison of different approaches.
\begin{table}[t]
\centering
\scalebox{0.78}{
\begin{tabular}{c|c|c|c}
\toprule
\textbf{N} & \textbf{Solvers} & \textit{Avg}. $ || R_{ij} - R_j {R_i}^T || $ \textcolor{green}{$\downarrow$} & \textit{Avg}. $|| \theta^* - \theta ||$ \textcolor{green}{$\downarrow$} \\ \midrule
\multirow{2}{*}{10} 
& Ours (iterative)        & \{\textbf{1.484} $\pm$ \textbf{0.08}\} \textbf{e-17}   & \{\textbf{0.933} $\pm$ \textbf{0.05} \} \textbf{e-17}    \\ 
& Ours (direct)  & 0.863 $\pm$ 0.07          & 0.826 $\pm$ 0.08           \\ \midrule
\multirow{2}{*}{15}  
& Ours (iterative)        & \{\textbf{1.156} $\pm$ \textbf{0.24}\} \textbf{e-17}   & \{\textbf{7.843} $\pm$ \textbf{0.87}\} \textbf{e-18}    \\ 
& Ours (direct)  & 1.226 $\pm$ 0.18           & 0.975 $\pm$ 0.17           \\ \midrule
\multirow{2}{*}{20} 
& Ours (iterative)        & \{\textbf{9.342} $\pm$ \textbf{0.86}\} \textbf{e-17}   & \{\textbf{6.685} $\pm$ \textbf{0.69}\} \textbf{e-17}    \\ 
& Ours (direct)  & 0.892 $\pm$ 0.065          & 0.793 $\pm$ 0.059           \\ \bottomrule
\end{tabular}
}
\caption{Benchmark results of IQARS for MRA on synthetic noise-free rotations.}
\label{noiseless}
\vspace{-5pt}
\end{table}
We show that, compared to ~\cite{golyanik2020quantum}, our design is empirically valid and achieves promising performance on D-Wave's annealing machines.

\noindent \textbf{Hyperparameter Study}. We perform a rigorous sensitivity analysis of hyperparameters in Alg.~\ref{IterativeQRA} on synthetic datasets via controlled experiments.
Specifically, we evaluate: (i) annealing threshold $\kappa$, (ii) annealing rate $\tau$, (iii) qubit count $m$ for discretization, and (iv) penalty coefficient $\alpha$.
All hyperparameters remained fixed post-optimization for controlled studies.
Our experimental results as presented in Fig.~\ref{ablation} (log scale) demonstrate that increasing $m$ at the cost of more quantum resource budget leads to monotonic improvements in convergence properties due to increased search interval representation fidelity; see Sec.~\ref{scalability}.
Reducing the annealing rate $\tau$ results in more stable optimization trajectories.
As the problem scales up, the convergence behavior tends to get progressively slower under otherwise identical experimental conditions.
Notably, the penalty factor $\alpha$ exhibits a Goldilocks effect—both excessive and insufficient penalization degrade performance, leaving the optimal $\alpha$ worth exploring.
This empirical observation is consistent with our theoretical analysis earlier in Sec.~\ref{SO(3)}.
%
These findings collectively establish principled guidelines for parameter selection in leveraging IQARS.

\noindent \textbf{Noise-Corrupted Synthetic Rotation Graphs}. Practically, observed real-world rotation measurements are inherently noisy.
We model this scenario and generate perturbed measurements $\tilde{R}_{ij}$ by applying controlled multiplicative distortions on the $\mathrm{SO}(3)$ manifold.
Specifically, we simulate such noise corruption by: (i) sampling a random tangent vector $\mathbf{v} \in \mathbb{R}^3$ uniformly from $[0,1]^3$, (ii) scaling it by a prescribed noise level $\sigma$, and (iii) mapping it to $\mathrm{SO}(3)$ via the exponential map $\exp(\sigma \mathbf{v})$.
This noise construction is consistent with previous approaches~\cite{wilson2016rotations,lee2024robustsinglerotationaveraging}, ensuring that the perturbed rotations remain properly constrained to the manifold while preserving the geometric Lie group structure.
\begin{figure}[t]
    \centering 
    \includegraphics[width = 0.48\textwidth]{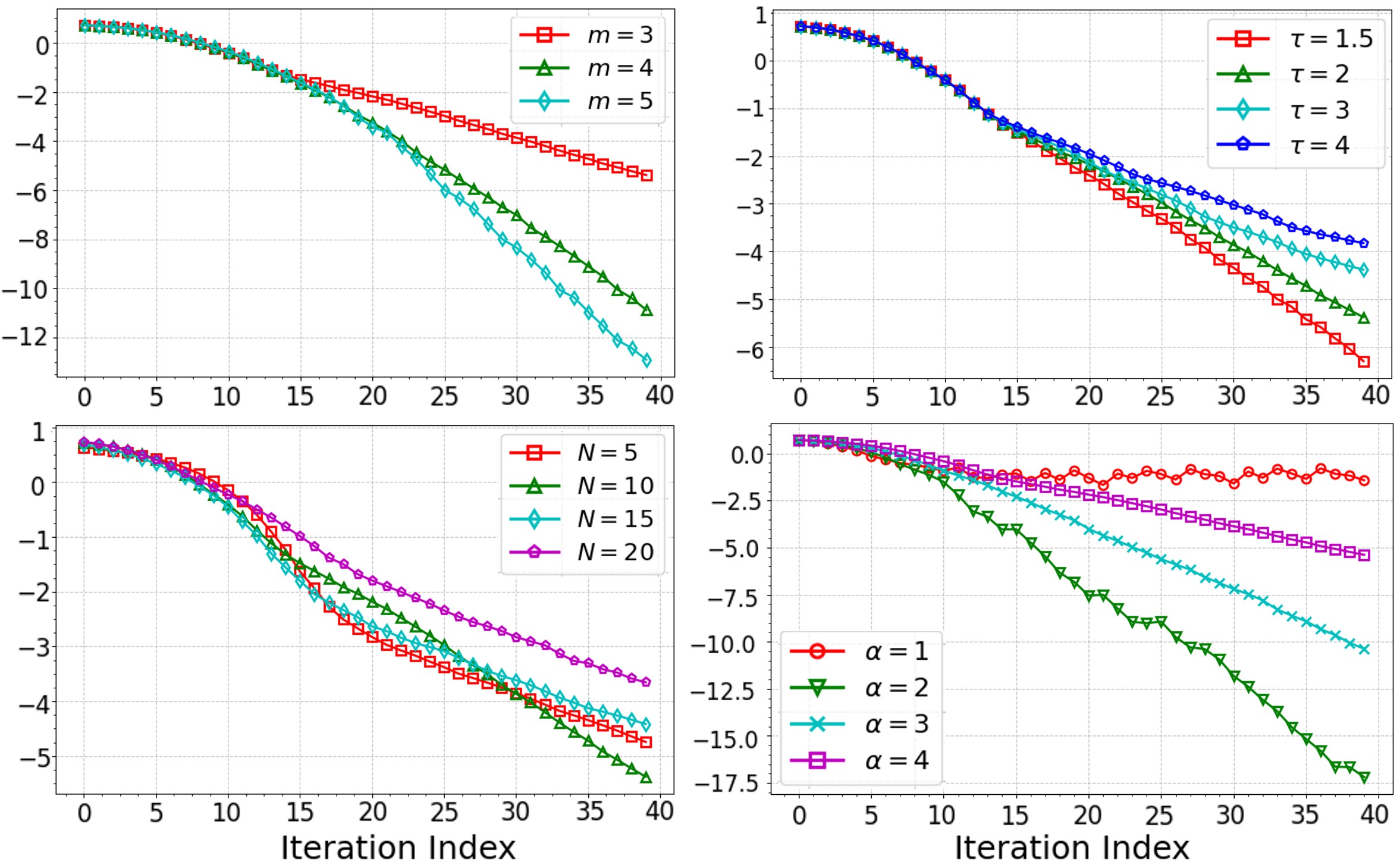} 
    \caption{Visualization of IQARS's hyperparameter choice on the convergence behavior (MRA residual on log scale on y-axis).}
    \label{ablation}
    \vspace{-5pt}
\end{figure}
We evaluate the performance of IQARS and provide additional experiments with established noise-robust classical solvers,
specifically the Riemannnian optimization framework of Shonan~\cite{dellaert2020shonan} and the robust L1-IRLS algorithm~\cite{chatterjee2013efficient}.
L1-IRLS leverages minimum spanning tree~(MST) initialization~\cite{chatterjee2013efficient,hartley2013rotation}, while Shonan employs random SO($p$) initialization; $p \ge 1$~\cite{dellaert2020shonan}.
We visualize in Fig.~\ref{noise} the performance of different MRA solvers on synchronizing $N=20$ rotations across increasing noise ratios $\sigma$.
\textit{While primarily designed to pioneer a novel quantum solver for MRA rather than challenging classical MRA solvers especially with available hardware of limited resource budget resulting in small problem sizes and coarse search space discretization}, empirical results highlight that IQARS on D-Wave machines can achieve residual reduction by approximately 12\% compared to the second best---Shonan, across a wide range of noise regimes, showing great potentials of our approach in future.
For completeness, we also: i) benchmark IQARS against two prominent local optimization methods---trust-region and Levenberg-Marquardt~(LM)---both theoretically and empirically; see details in App.~\ref{s:local_solvers} and ii) evaluate IQARS (identity initialization by default to avoid heuristic bias) with other initializations; see App.~\ref{s:init}.

\begin{figure}[t]
    \centering 
    \includegraphics[width = 0.48\textwidth]{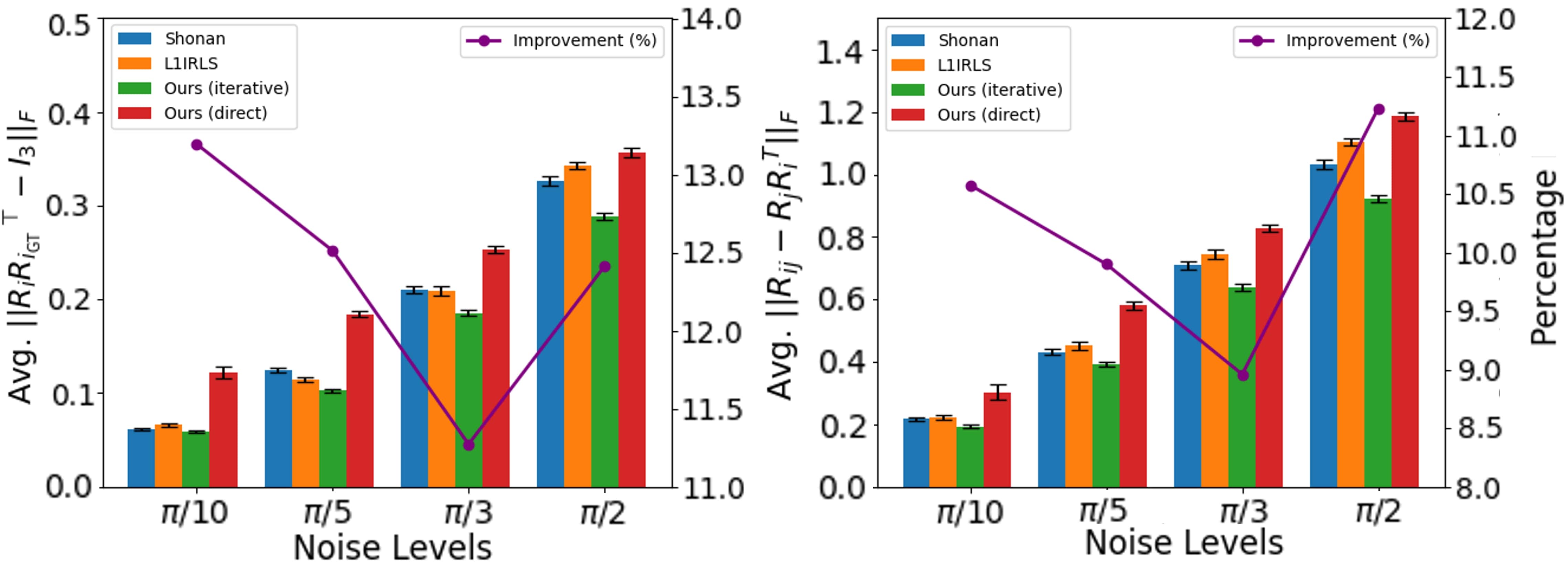} 
    \caption{Benchmark results of IQARS against other solvers for MRA on a synthetic noisy dataset across noise levels $\sigma$; ${R_{i}}_{GT}$ represents the ground truth of $R_i$.} 
    \label{noise}
    \vspace{-5pt}
\end{figure}

\noindent \textbf{Real-World Noisy Dataset Evaluation}. To validate the practical utility of IQARS beyond synthetic experiments, we evaluate its performance alongside Shonan and L1-IRLS on established real-world benchmark datasets~\cite{strecha2008benchmarking}.
Results are documented in Tab.~\ref{results}, demonstrating that IQARS maintains robustness and has performance on par with or outperforms other classical solvers on real-world datasets.
This observation is consistent with our previous experiments on synthetic datasets.
The chain break ratio, as a proxy for IQARS embedding quality within the quantum hardware topology, is also monitored and reported in Tab.~\ref{results}.
\begin{table}[htbp]
\vspace{-5pt}
\centering
\scalebox{0.57}{
\begin{tabular}{c|c|c|c|c|c|c}
\toprule
N & Dataset & Evaluation Metrics & Ours (iterative) & Shonan & L1-IRLS & Ours (direct) \\ \midrule
\multirow{3}{*}{11} & \multirow{3}{*}{Fountain} & \textit{Avg}. $||R_{ij} - R_j{R_i}^T||$ \textcolor{green}{$\downarrow$} & \textbf{3.72e-3} & 4.18e-3 & 4.39e-3 & 8.17e-2 \\
& & \textit{Avg}. $||\theta^* - \theta||$ \textcolor{green}{$\downarrow$} & \textbf{3.51e-3} & 4.04e-3 & 4.23e-3 & 6.73e-2 \\
& & \textit{Avg}. chain break (\%) \textcolor{green}{$\downarrow$} & 6.2e-2 & / & / & 5.81e-2 \\ \midrule
\multirow{3}{*}{15} & \multirow{3}{*}{Castle} & \textit{Avg}. $||R_{ij} - R_j{R_i}^T||$ \textcolor{green}{$\downarrow$} & \textbf{1.12e-3} & 1.33e-3 & 1.65e-3 & 7.65e-2 \\
& & \textit{Avg}. $||\theta^* - \theta||$ \textcolor{green}{$\downarrow$} & \textbf{1.02e-3} & 1.19e-3 & 1.48e-3 & 5.65e-2 \\
& & \textit{Avg}. chain break (\%) \textcolor{green}{$\downarrow$} & 1.13e-1 & / & / & 1.02e-1 \\ \midrule
\multirow{3}{*}{8} & \multirow{3}{*}{Herz-Jesus} & \textit{Avg}. $||R_{ij} - R_j{R_i}^T||$ \textcolor{green}{$\downarrow$} & \textbf{3.26e-3} & 3.89e-3 & 4.06e-3 & 9.86e-2 \\
& & \textit{Avg}. $||\theta^* - \theta||$ \textcolor{green}{$\downarrow$} & \textbf{3.08e-3} & 3.62e-3 & 3.88e-3 & 8.25e-2 \\
& & \textit{Avg}. chain break (\%) \textcolor{green}{$\downarrow$} & 3.8e-2 & / & / & 4.2e-2 \\ \bottomrule
\end{tabular}
}
\caption{Benchmark results of IQARS against other solvers for MRA on real-world datasets~\cite{strecha2008benchmarking}; hardware statistics such as chain break ratios reflecting machine embedding quality are reported.}
\label{results}
\end{table}
As the problem scales up, i.e. larger $N$, the chain break ratio is observed to slightly increase.
While this (at a low level) does not seem to significantly impact the performance, this reflects an inherent limitation of embedding larger problems into existing hardware topology.
Lastly, we also integrated IQARS within the Glomap, a SfM pipeline~\cite{pan2024glomap}; see App.~\ref{SfM_pipeline}, for a qualitative visual analysis of the Poisson-reconstructed meshes. 
The visualization in Fig.~\ref{reconstructed} shows accurate 3D reconstructions, establishing IQARS as a robust MRA method with potentials for applications such as SfM.

\noindent \textbf{Posterior Analysis}. We rigorously evaluate our posterior protocol that aims to relieve the thermal noise effect during annealing and improve the solution quality; see Sec.~\ref{sec:postanal}.
The protocol incorporates low-energy solution spaces of various spans.
We define and maintain inverse temperature $\beta=2$---carefully chosen to empirically balance exploration/exploitation trade-offs---and systematically vary the span $K$ of low-energy solution spaces.
We observe that the proposed posterior protocol consistently improves solution quality and yields statistically significant gains over the unrefined annealer outputs.
Detailed statistical results with different variable spans $K$ of the solution space are recorded and visualized in Fig.~\ref{posterior}.
The visualization additionally shows a characteristic saturation curve: with initial increases in $K$ yielding substantially improved performance, the marginal improvement diminishes asymptotically beyond an optimal value, i.e. a sweet spot, for low-energy sample space selection.
This sweet spot for the MRA is empirically observed to be around $K = 30$, suggesting optimal sample space selection.

\begin{figure}[t]
    \centering 
    \includegraphics[width = 0.48\textwidth]{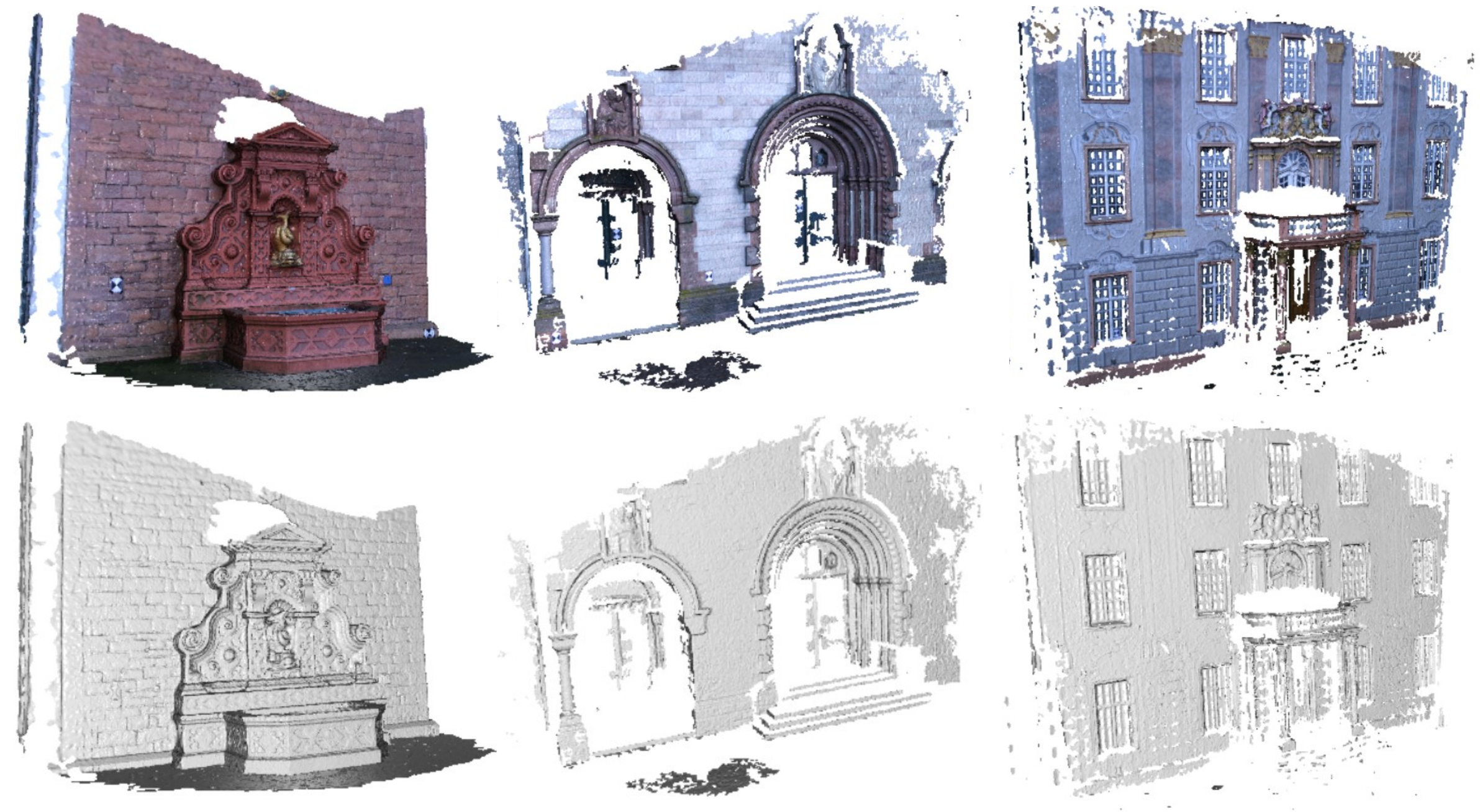} 
    \caption{Reconstructed 3D geometries from multi-view imagery via Glomap~\cite{pan2024glomap} integrated with IQARS.
    Both color-textured (top) and geometry-only (bottom) representations of three benchmark datasets—\textit{Fountain}, \textit{Castle}, and \textit{Herz-Jesus} (left-to-right) are visualized.
    Quantitative performance metrics are detailed in Tab.~\ref{results}.
    } 
    \label{reconstructed}
    \vspace{-15pt}
\end{figure}

\noindent \textbf{Embedding on Contemporary Annealers}. Alongside theoretical scalability analysis in Sec.~\ref{scalability} that characterize \textit{logical qubit} requirements of IQARS, we evaluate its \textit{physical resource} demands by considering problem embeddings on the hardware topology of D-Wave Advantage System (Pegasus architecture).
Specifically, we measure the physical resource consumption for MRA across varying problem sizes $N$ and visualized it in Fig.~\ref{q_resources} in App.; it also includes other machine versions.  
As an example, for synchronizing $N = 15$ rotations requiring $135 = 3Nm$ logical qubits, approximately 2500 physical qubits are required to embed the logical problem on the hardware topology of the Advantage 6.4 annealer.
This accounts for nearly half of the processor's total capacity (5000 physical qubits).
\textit{While IQARS is designed to be resource-efficient, its practical implementation for real-world tasks on contemporary annealer hardware can present challenges that constrain its performance.}

\begin{figure}[t]
    \centering 
    \includegraphics[width = 0.48\textwidth]{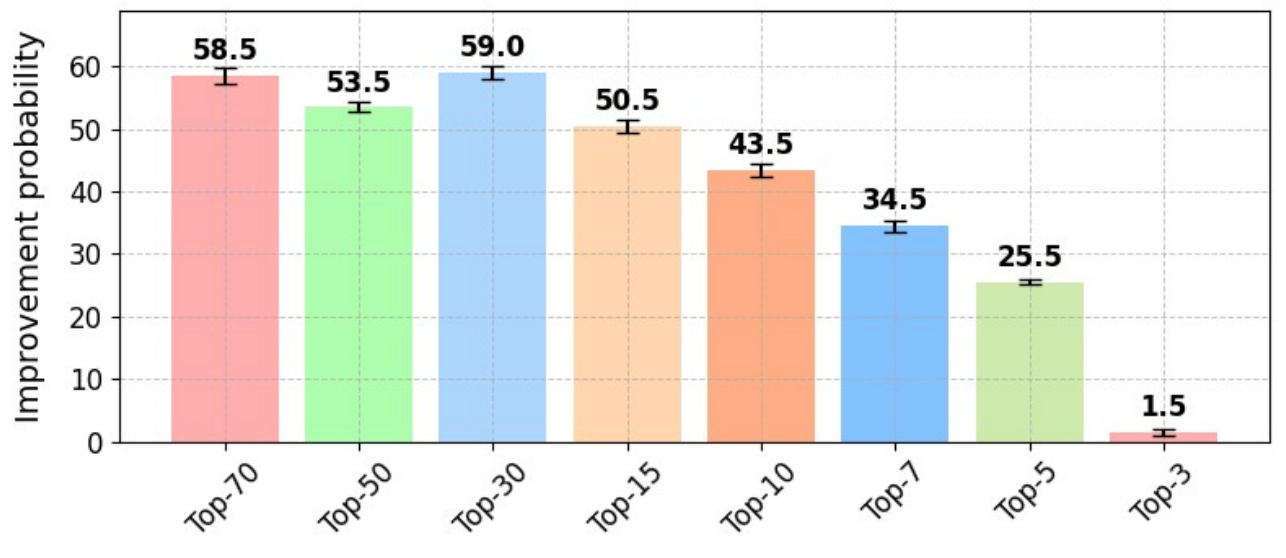} 
    \caption{Statistical improvements due to our refinement protocol with variable span $K$.
    Top-$K$ refers to posterior aggregation performed over the $K$ lowest-energy annealer samples.} 
    \vspace{-5pt}
    \label{posterior}
\end{figure}

\section{Discussion and Conclusion} \label{D_and_C} 

\noindent We introduced IQARS, the first framework for MRA leveraging quantum hardware.
While not positioned as a replacement for classical approaches in the near future, IQARS establishes a novel computational paradigm exploration for MRA.
Despite a limited quantum resource budget resulting in constrained performance, extensive experiments demonstrate that: i) IQARS can recover ground-truth solutions for MRA with a noise-free synthetic dataset; ii) for noisy datasets (synthetic and real-world), a consistent $10$--$15\%$ residual reduction is observed for IQARS compared to Shonan.
With more resources accessible leading to e.g. finer search discretization, the performance of IQARS is expected to improve significantly as per our hyperparameter study.
Our posterior refinement protocol statistically improves sampled solution quality, achieving up to a 60\% probability of producing lower-residual solutions compared to the unrefined annealer outputs.
We aim to establish foundations for advances in 3D vision that build upon ours in Quantum-enhanced Computer Vision (QeCV), and will release source codes for reproducibility.

\noindent \textbf{Limitations}. While IQARS, as a resource-efficient quantum MRA algorithm, demonstrates considerable promise, the experimental scale and its potential for large-scale real-world problems remain constrained with limited resource budget from contemporary annealer hardware.

\noindent \textbf{Future Work}. IQARS holds significant potential for diverse downstream applications.
Beyond SfM, promising tasks include robotic navigation, medical imaging and augmented reality, among others.
Further investigation into alternative MRA formulations with different hardware support presents a compelling avenue for future research.

%% file: sec/suppl.tex
\clearpage
\appendix
\maketitlesupplementary

\setlength{\intextsep}{5pt}
\setlength{\columnsep}{5pt}

\noindent This appendix supplements the main paper with expanded content, starting with a detailed background on MRA in Sec.~\ref{sec:preliminaries}.
Explanation of the Hamiltonian evolution during quantum annealing is provided in Sec.~\ref{evolution_annealing}.
Sec.~\ref{AQC} explains the adiabatic theorem.
Sec.~\ref{proof_p1} and \ref{proof_p3} provide details of the propositions in the main text.
The effect of penalization is explained in Sec.~\ref{penalization_effect}.
Sec.~\ref{QUBO} details the transformation of the QUBO problem into the Ising form, a Hamiltonian representation more naturally aligned with the physical interactions implemented in quantum annealers.
In Sec.~\ref{ours:direct}, we explain how Prop.~\ref{proposition1} can be combined with \cite{golyanik2020quantum}, leading to an alternative direct MRA approach.
Sec.~\ref{posterior_protocol} details the posterior protocol for improving sample quality.
We introduce details of integrating IQARS, as a module for averaging multiple rotations, in an SfM pipeline; see Sec.~\ref{SfM_pipeline}.
%
%
In addition to synchronizing fully-connected camera rotations, Sec.~\ref{sparse_graph} provides evaluations on sparsely-connected camera rotations.
%
%
Visualizations of the logical coupling matrix and its embeddings on the hardware are provided in Sec.~\ref{QUBO_sparsity} and \ref{embedding_analysis}.
An additional comparative analysis between IQARS and a few other prominent local solvers is provided in Sec.~\ref{s:local_solvers}.
Evaluations of IQARS with different initializations are provided in Sec.~\ref{s:init}.

\section{Review: MRA}
\label{sec:preliminaries}

\noindent This section reviews the basic definition of MRA. 
Let $\left\{ R_1, \cdots,  R_N\right\} \subset \text{SO} (3)$ represent a set of absolute $3$D rotational matrices w.r.t.~a global reference frame. 
The special orthogonal group $\text{SO} (3) $ is the set of matrix elements $R$ that are orthogonal and have a unit determinant: 
\begin{equation}
    \text{SO} (3) := \left\{ R \in \mathbb{R}^{3 \times 3}, \ \  R  R ^T =  I , \ \mathrm{det}( R) = 1\right\}.
\end{equation}
Given the set of absolute rotations $\left\{ R_1, \cdots,  R_N\right\}$, the relative rotation \(R_{ij}\) between any two cameras \(i\) and \(j\), where \(i, j \in \{1, \cdots, N\}\), can be computed as 
\begin{equation}
    R_{ij} = R_j R_i^{-1} \; \text{or, equivalently}, \; R_{ij} R_i = R_j.
\end{equation} 
Assume we are given a set of noisy relative rotations $\{ \tilde{R}_{ij} \mid i, j = 1, \dots, N\}$ between different input camera pairs $i, j$, the objective of MRA is to find the absolute rotations $\{ R_i \mid i = 1, \dots, N\}$ that best explains the observed $\tilde{R}_{ij}$. 
%
MRA is an inverse optimization problem 
to minimize the discrepancy between the predicted and observed relative rotations:
\begin{equation}\label{RA_eqn}
    \min_{ R_1, \dots,  R_N \in \text{SO}(3)} \sum_{(i,j)} \operatorname{dist}( \tilde{ R}_{ij}  R_i,  R_j),
\end{equation}
where ``\(\operatorname{dist}(a, b)\)'' is a chosen distance metric that quantifies the difference between two rotations. 
With the widely used choice of squared chordal distance, i.e.,~$\text{dist}(a, b) = \|a - b \|_F ^2$, we can re-formulate Eq.~\eqref{RA_eqn} as
\begin{equation} \label{RA_final}
    \min_{ R_1,  \dots,  R_N \in \text{SO}(3)} \sum_{(i,j)} \| \tilde{ R}_{ij}  R_i -  R_j \|_F^2.
\end{equation}
Eq.~\eqref{RA_final} provides the starting MRA formulation, which we map to a quantum-compatible form in the main text.

\section{Review: Quantum System Evolution during Annealing} \label{evolution_annealing}

\noindent The evolution of quantum states during the Hamiltonian transition in the annealer is governed by Schrödinger's equation
\begin{equation}
    i \hbar \frac{\partial}{\partial t} \ket{\psi(t)} = H(t) \ket{\psi(t)},
\end{equation}
where \( i \) is the imaginary unit, \( \hbar \) denotes the reduced Planck constant, \( \ket{\psi(t)} \) represents the quantum state at time \( t \), and \( H(t) \) is the time-dependent Hamiltonian.
Its analytical solution can be expressed as
\begin{equation}
    \ket{\psi(t)} = \hat{T} \text{exp}\big({-\frac{i}{\hbar} \int_{0}^{t} H(t')dt' }\big) \ket{\psi(0)},
\end{equation}
where $\hat{T}$ is the time ordering operator, ensuring chronological sequence in the exponential expansion. By substituting in the time-dependent Hamiltonian, see Eq.~\ref{eq:time_dependent_hamiltonian}, we can, therefore, describe the evolution of quantum states during the annealing process as follows:
\begin{multline}
    \ket{\psi(t)} = \hat{T} \text{exp} \left( -\frac{i}{\hbar} \int_{0}^{s} \left[ -A(s') \sum_{i} \sigma_i^x \right. \right. \\
    \left. \left. + B(s') \left( \sum_{i,j} J_{ij} \sigma_i^z \sigma_j^z + \sum_i h_i \sigma_i^z \right) \right] ds' \right) \ket{\psi(0)}.
\end{multline}
\( A(s') \) and \( B(s') \) are annealing schedules.
\( \sigma_i^x \) and \( \sigma_i^z \) are Pauli spin operators.
\( J_{ij} \) denotes coupling strengths, and \( h_i \) represents local magnetic fields.
These collectively define the system's quantum dynamics.

\section{Review: Adiabatic Theorem} \label{AQC}

\noindent Adiabatic quantum computation~(AQC) represents a specialized framework within quantum computing that leverages the adiabatic theorem~\cite{albash2018adiabatic,jansen2007bounds,kadowaki1998quantum,born1928beweis}. Theorem~\ref{thm:adiabatic} formalizes the key ideas of this principle:

\begin{theorem} \label{thm:adiabatic}
Adiabatic Theorem: A physical system remains in its instantaneous eigenstate if a given perturbation is acting on it slowly enough and if there is a gap between the initial eigenvalue and the rest of the Hamiltonian's spectrum.
\end{theorem}

\section{Proof of Proposition 1} \label{proof_p1}

\noindent We provide proof of the Prop.~\ref{proposition1} given in the main text: 
\vspace{-4pt}
\begin{equation}
    \begin{aligned}
        &\min_{R_1, \dots, R_N} \sum_{(i,j)} \| \tilde{R}_{ij} R_i - R_j \|_F^2 \\ = &\min_{R_1, \dots, R_N} \sum_{(i,j)} \| \tilde{R}_{ij} \|_F^2 \| R_{i} \|_F^2 + \| R_{j} \|_F^2 - 2 \langle R_{ij} R_{i}, {R}_{j}  \rangle _F \\
        = &\min_{{R}_1, \dots, {R}_N} \sum_{(i,j)} 3 \| {R}_{i} \|_F^2 + \| {R}_{j} \|_F^2 - 2 \mathrm{Tr}( {R}_{i} ^\top {R}_{ij} ^\top {R}_{j}  \rangle \\
        = &\min_{{R}_1, \dots, {R}_N} \sum_{(i,j)} 3 \| {R}_{i} \|_F^2 + \| {R}_{j} \|_F^2 - 2 \mathrm{Tr}( {R}_{j} {R}_{i} ^\top {R}_{ij} ^\top \rangle \\
        = &\min_{{R}_1, \dots, {R}_N} \sum_{(i,j)} 3 \| {R}_{i} \|_F^2 + \| {R}_{j} \|_F^2 - \\
        &2 \vect({R}_i) ^\top (I \otimes {R}_{ij} ^\top) \vect( {R}_j).
    \end{aligned}
\end{equation}

Given the orthonormal nature of rotation matrices \( {R}_i \in \SO(3) \), they satisfy the fundamental properties \( {R}_i^\top = {R}_i^{-1} \) and \( \det({R}_i) = 1 \). These constraints imply that the Frobenius norm \( \|{R}_i\|_F \) is always \( \sqrt{3} \) for any rotation matrix and can be eliminated from the optimization objective. 
We reformulate the quadratic alignment term \( \sum_{(i,j)} -2 \vect({R}_i)^\top (I \otimes {R}_{ij}^\top) \vect({R}_j) \) in compact matrix form by eliminating the explicit summation through Kronecker product identities and rewrite it as

\begin{equation}
    - 2 \left[\vect({R}_1)^\top \cdots \vect({R}_N)^\top \right] {P} \left[
    \begin{matrix}
    \vect({R}_1) \\
    \vdots \\
    \vect({R}_N)
    \end{matrix}
    \right],
\end{equation}
with $\vect({R}_i)^\top \in \mathbb{R}^{1\times 9}$. The corresponding ${P}$ is 
\begin{equation}
    {P} = \begin{bmatrix}
        I \otimes {R}_{11} ^\top & \cdots & I \otimes {R}_{1N} ^\top \\
        \vdots & \ddots & \vdots \\ I \otimes {R}_{N1} ^\top & \cdots & I \otimes {R}_{NN} ^\top
    \end{bmatrix} \in \mathbb{R}^{9N \times 9N}.
\end{equation}

\section{Penalization Effect} \label{penalization_effect}

\noindent While each iterative update of IQARS preserves ${R}_i \in \SO(3)$ through exponential map parameterization, the local search window—defined by first-order linearization of the tangent space—permits bounded exploration of points outside $\SO(3)$. 
Within this region, solutions with larger \( \|{R}_i\|_F \) are favored due to the negative coupling term (see Eq.~\ref{Q_cal}) in the objective function.
This exacerbates the tendency of solutions to drift away from the $\SO(3)$ manifold constraints.
We employ \( \|{R}_i\|_F^2 \) as a principled metric for quantifying geometric deviation from \(\text{SO}(3)\) during local search, with an appropriately chosen regularization strength \(\lambda\) counteracting this tendency within the bounded exploration region.

\section{Proof of Proporsition 3} \label{proof_p3}

\noindent We provide proof of the Prop.~\ref{proporsition3} as given in the main text: 
\begin{equation}
    \begin{aligned}
    & \argmin_{\|\Delta \bm{v}\|_\infty \leq \delta^k} \, \Delta \bm{v}^\top \hat{Q} \Delta \bm{v} + \hat{c}^\top \Delta \bm{v} \\
    = & \argmin_{\bm{q} \in \B^{3Nm}} \left( -\delta^k {1}_{3N} + D \bm{q} \right)^\top \hat{Q} \left( -\delta^k {1}_{3N} + D \bm{q} \right) \\
    & + \hat{c}^\top \left( -\delta^k {1}_{3N} + D \bm{q} \right) \\
    = & \argmin_{\bm{q} \in \B^{3Nm}} \bm{q}^\top D^\top \hat{Q} D \bm{q} + \bm{q}^\top D^\top \hat{Q} \left( -\delta^k {1}_{3N} \right) \\
    & + \left( -\delta^k {1}_{3N} \right)^\top \hat{Q} D \bm{q} + \hat{c}^\top (D \bm{q}) \\
    = & \argmin_{\bm{q} \in \B^{3Nm}} \bm{q}^\top D^\top \hat{Q} D \bm{q} + \bm{q}^\top D^\top \left( \hat{\bm{c}} - 2 \delta^k \hat{Q} {1}_{3N} \right).
    \end{aligned}
\end{equation}

\section{Transformation of QUBO Problems to Ising-Compatible Quadratic Forms} \label{QUBO}

\noindent Given the general QUBO formulation
\begin{equation}
    \min_{\bm x \in \{0, 1\}^n} \bm x^\top Q \bm x + \bm b^\top \bm x,
\end{equation}
we require a transformation to the standard quadratic form for implementation on quantum annealers employing Ising Hamiltonians. Leveraging the Boolean constraint $\bm x \in \{0, 1\}^n$ (which implies $x_i^2 = x_i$), we derive the equivalent pure quadratic form through the following sequence:
\begin{equation}
    \begin{aligned}
        \min_{\bm x \in \{0, 1\}^n} \bm x^\top Q \bm x + \bm b^\top \bm x 
        &= \min_{\bm x \in \{0, 1\}^n} \bm x^\top Q \bm x + \bm x^\top \mathrm{diag}(\bm b) \bm x \\
        &= \min_{\bm x \in \{0, 1\}^n} \bm x^\top \left(Q + \mathrm{diag}(\bm b)\right) \bm x \\
        &\triangleq \min_{\bm x \in \{0, 1\}^n} \bm x^\top Q' \bm x,
    \end{aligned}
\end{equation}
where the transformed matrix $Q' \in \mathbb{R}^{n \times n}$ is constructed as
\begin{equation}
    Q' = Q + \mathrm{diag}(\bm b).
\end{equation}
This exact reformulation: (i) preserves the original optimization problem's solution space, (ii) maintains compatibility with physical Ising model implementations through the identity $x_i = x_i^2$, and (iii) enables efficient embedding on quantum annealers by consolidating all terms into quadratic couplings.

\section{Alternative Approach for MRA} \label{ours:direct}

We also realized that in combination with Prop.~\ref{proposition1}, the approach proposed in \cite{golyanik2020quantum} can be an alternative for MRA.
Specifically, they propose to directly approximate rotation matrices $R_i \in \SO(3)$ through a linear combination of binary-activated basis matrices:
\begin{equation}
\label{eq:approx_vlad}
    R_i \approx \sum_{\ell=0}^{m-1} q_{i,\ell} Q_{\ell}, \quad q_{i,\ell} \in \{0, 1\},
\end{equation}
where the basis matrices $\{Q_{\ell}\}$ quantize the components of the Rodrigues rotation formula; see Eq.~\eqref{eq:exponentialmap}. Specifically, the basis set comprises scaled versions of the identity matrix and generators of $\mathfrak{so}(3)$:
\begin{multline}
\label{eq:basis_vlad}
    Q_{\ell} \in \big\{wC \mid w \in \{0.5, 0.2, 0.1, 0.1, 0.05\}, \\
    C \in \{I, -I\} \cup \{M_k, -M_k\}_{k=1}^6 \big\},
\end{multline}
with the generator matrices defined as
\begin{equation*}
    M_1 = \scalebox{0.8}{$
    \Matrix{
    0 & 0 & 0 \\
    0 & 0 & -1 \\
    0 & 1 & 0
    }
    $},
    M_2= \scalebox{0.8}{$
    \Matrix{
    0 & 0 & 1 \\
    0 & 0 & 0 \\
    -1 & 0 & 0
    }
    $},
    M_3 = \scalebox{0.8}{$
    \Matrix{
    0 & -1 & 0 \\
    1 & 0 & 0 \\
    0 & 0 & 0
    }
    $}
\end{equation*}
\begin{equation}
    M_4 = \scalebox{0.8}{$\Matrix{
    0 & 1 & 0 \\
    1 & 0 & 0 \\
    0 & 0 & 0
    }$},
    M_5 = \scalebox{0.8}{$\Matrix{
    0 & 0 & 1 \\
    0 & 0 & 0 \\
    1 & 0 & 0
    }$},
    M_6 = \scalebox{0.8}{$\Matrix{
    0 & 0 & 0 \\
    0 & 0 & 1 \\
    0 & 1 & 0
    }$}.
\end{equation}
Substituting the approximation from Eq.~\eqref{eq:approx_vlad} into the matricized problem (see Eq.~\eqref{eq:matricization} in Prop.~\ref{proposition1}) yields a QUBO formulation that directly approximates the MRA problem.
However, this approach presents two significant theoretical limitations: (1) the finite basis set $\{Q_{\ell}\}$ imposes fundamental approximation bounds due to its limited expressiveness, and (2) the resulting matrices are not guaranteed to satisfy the orthonormality conditions ($R_i^\top R_i = I$) or the determinant constraint ($\det(R_i) = 1$) required for proper $\text{SO}(3)$ membership, as these nonlinear constraints are not explicitly enforced in the binary formulation.
In our IQARS, such constraints are resolved.

\begin{figure*}[t] 
    \centering 
    \includegraphics[width=1.0\textwidth]{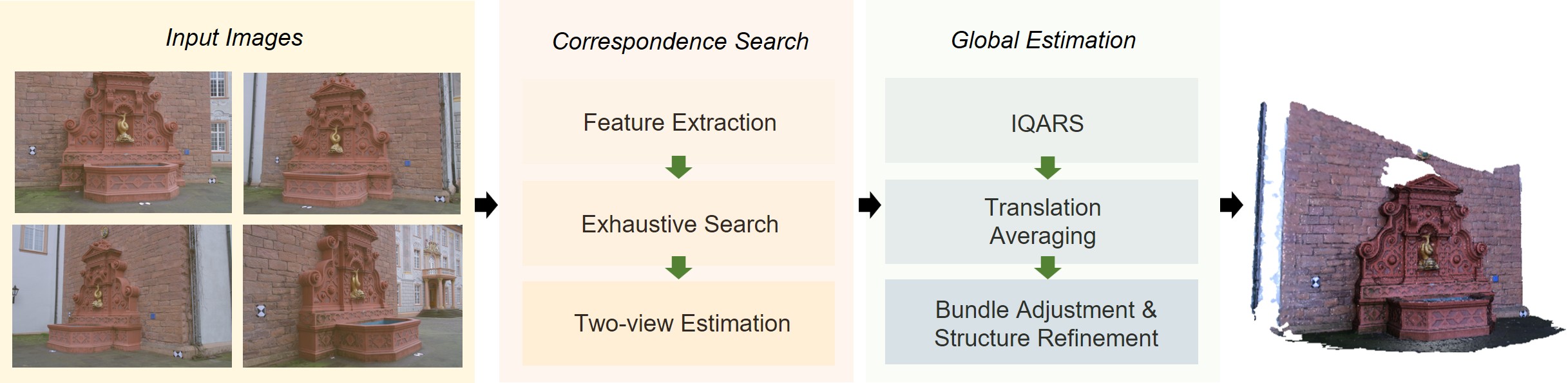} 
    \caption{Integration of our \textit{IQARS} protocol within a traditional SfM pipeline.
    Our IQARS replaces the traditional MRA component while maintaining compatibility with the existing pipeline.}
    \label{Quantum_enhanced_SfM_pipeline} 
\end{figure*}

\section{Posterior Analysis Protocol} \label{posterior_protocol}
\noindent Algorithmic details of performing posterior analysis to refine solution qualities are summarized in Alg.~\ref{Refine_protocol}.

\begin{algorithm}[ht] 
\caption{Posterior Refinement Protocol with Low-energy Binary Solution Space} \label{Refine_protocol}
\begin{algorithmic}[1]
    \State \textbf{Input:} $\textbf{x}$: Low-energy binary strings from annealers
    \Statex \hspace {2.8em} $E(\textbf{x})$: Energies of prepared binary strings
    \Statex \hspace {2.8em} $\beta$: Inverse temperature
    \Statex \hspace {2.8em} $L$: Length of each binary string
    \Statex \hspace {2.8em} $M$: Number of binary strings
    \State Construct an empty array $\textbf{S} = [0, \cdots, 0]$ of size $L$
    \State Construct an empty array $\textbf{m} = [0, \cdots, 0]$ of size $L$
    \State Calibrate energy spectrum to a standardized range
    \State \textbf{For} $i = 1$ to $M$
    \State \hspace{1.5em} Calculate Boltzmann weighting factor $P(\textbf{x}_i)$ according to Eq.~\eqref{Boltzmann}
    \State \hspace{1.5em} \textbf{For} $j = 1$ to $L$
    \State \hspace{2.8em} \textbf{If} $\textbf{x}_i[j] = 1$
    \State \hspace{3.6em} $S[j] \leftarrow S[j]$ + $P(\textbf{x}_i)$
    \State \hspace{2.8em} \textbf{If} $\textbf{x}_i[j] = 0$
    \State \hspace{3.6em} $S[j] \leftarrow S[j]$ - $P(\textbf{x}_i)$
    \State \hspace{1.5em} \textbf{End}
    \State \textbf{End}

    \State \textbf{For} $k = 1$ to $L$
    \State \hspace{1.5em} \textbf{If} $\textbf{S}[j] > 0$
    \State \hspace{2.8em} $\textbf{m}[j]$ = 1
    \State \hspace{1.5em} $\textbf{If}$ $\textbf{S}[j] < 0$
    \State \hspace{2.8em} $\textbf{m}[j]$ = 0
\State \Return $\textbf{m}$, $\textbf{S}$
\end{algorithmic}
\end{algorithm}

\section{Application: IQARS within SfM} \label{SfM_pipeline}

\noindent SfM is a fundamental technique that reconstructs 3D scene geometry and camera poses from a collection of 2D images.
Within this framework, MRA serves as a critical synchronization step that enforces global consistency among relative orientation estimates.
We implement a hybrid quantum-classical pipeline by replacing the conventional MRA module in Glomap~\cite{pan2024glomap}—a state-of-the-art SfM system—with our IQARS.
The complete flow is described in Fig.~\ref{Quantum_enhanced_SfM_pipeline}
This hybrid pipeline maintains compatibility with conventional feature matching and downstream applications.
Quantitative visualization of the reconstructed 3D scenes is provided in Sec.~\ref{sec:experiments}.

\section{MRA for Sparse Problems} \label{sparse_graph}

\begin{wrapfigure}{r}{0.27\textwidth}
    \centering
    \includegraphics[width=0.25\textwidth]{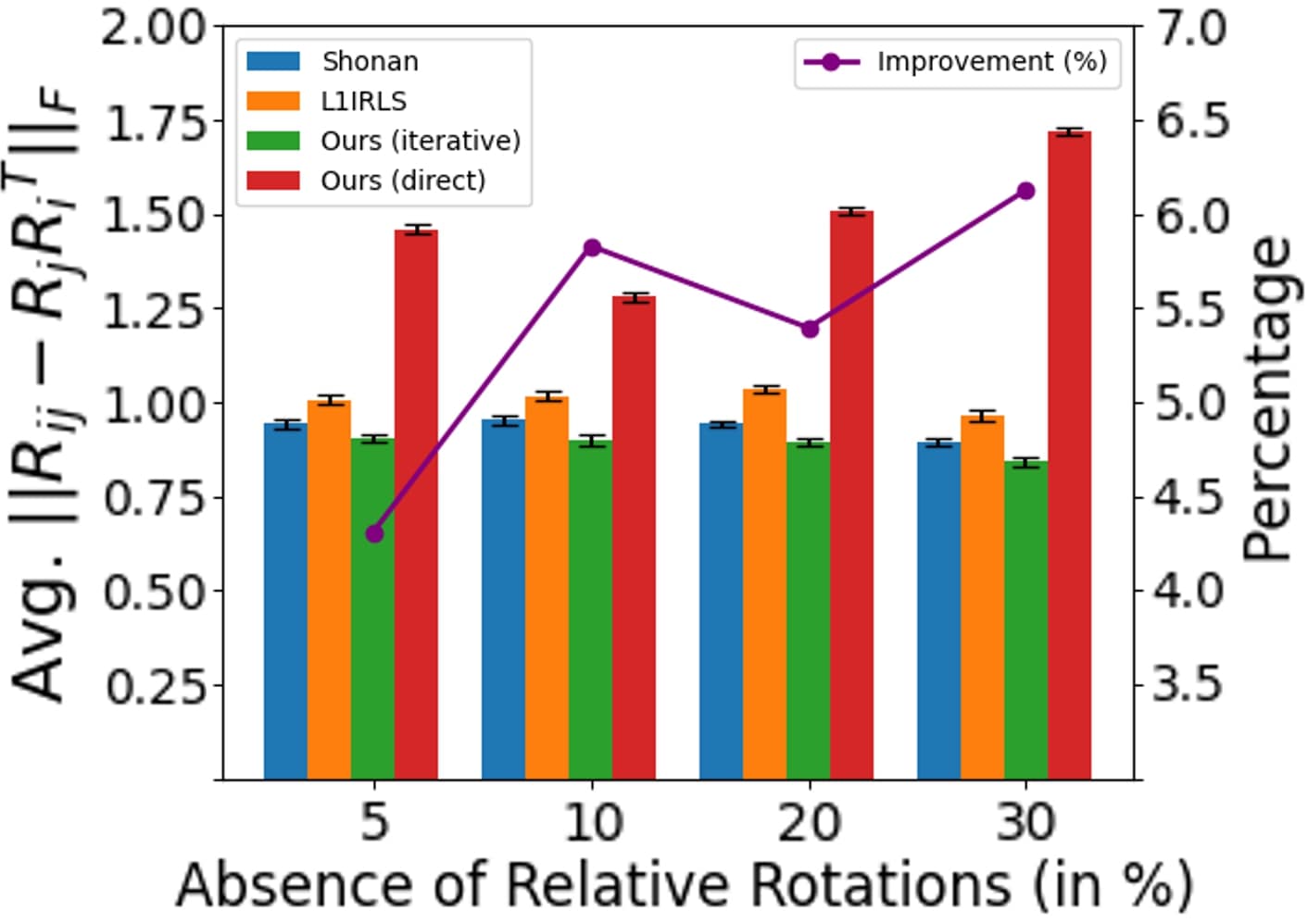} 
    \caption{Benchmark results of IQARS against other solvers for MRA on synthetic dataset with increasing camera graph sparsity.}
    \label{sparsity}
    \vspace{-7pt}
\end{wrapfigure} Building upon our theoretical framework and experimental validation for complete camera graphs $G = (V, E)$ where vertices $v_i \in V$ correspond to absolute camera orientations $R_i \in \mathrm{SO}(3)$, we generalize our analysis to sparse graph topologies characterized by incomplete edge sets $E' \subset E$.
The underlying mathematical formulation maintains consistency, where missing relative rotations $\tilde{R}_{ij}$ are treated as null observations and treated as zero terms during optimization.
To preserve the minimum connectivity requirement, we ensure that each node maintains at least one outgoing edge.
We leverage the same synthetic noisy dataset as in main experiments; see Sec.~\ref{sec:experiments}, and perform experiments by randomly switching off a certain percentage of relative observations.
We progressively increase the sparsity of observations and visualize the results in Fig.~\ref{sparsity}.
Results confirm that our algorithm maintains robustness across different sparsity ratios.

\section{QUBO Coupling Sparsity Visualization} \label{QUBO_sparsity}

\begin{figure}
    \centering 
    \includegraphics[width = 0.48\textwidth]{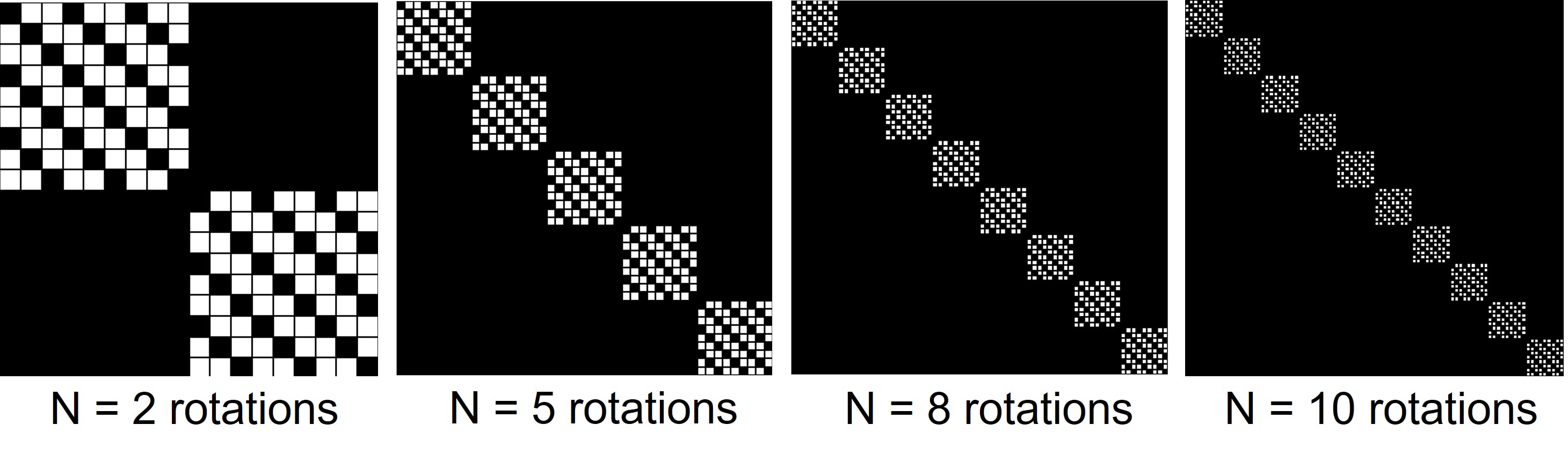} 
    \caption{Qualitative coupling matrix sparsity pattern visualization with increasing problem scale $N$ for MRA.} 
    \label{coupling_matrix} 
\end{figure}

\noindent We visualized the sparsity of the coupling matrix in the QUBO formulation before it gets uploaded to the annealer for execution.
The sparsity of the problem formulation determines how logical qubits---which represent the MRA solution update in IQARS---should interact with each other during the annealing process.
It also concerns the problem's embeddability on quantum hardware.
Typically, problems with a certain sparsity can admit efficient hardware mappings that reduce physical resource requirements and operational overhead.
We visualized the sparsity pattern of MRA that is employed in our algorithm in Fig.~\ref{coupling_matrix} with increasing problem sizes.
This can provide valuable guidelines for hardware embedding analysis.

\section{Visualizations of QPU (Pegasus Topology) Embedding} \label{embedding_analysis}

\noindent For completeness, we provide a qualitative visualization of the problem embeddings on the annealer of Pegasus topology; see Fig.~\ref{embedding_structure}.
Specifically, Fig.~\ref{embedding_structure}-(A) visualizes the hardware topology of the annealer.
Fig.~\ref{embedding_structure}-(B) to (D) visualizes the embedded MRA problems under the specific hardware topology.
The process of finding such a problem embedding is called \textit{minor embedding}.
Compared to the previous-generation Chimera architectures, the Pegasus topology enables more efficient embeddings of complex optimization problems and supports solving problems of larger scales.

\begin{figure*}[t] 
    \centering 
    \includegraphics[width=1.0\textwidth]{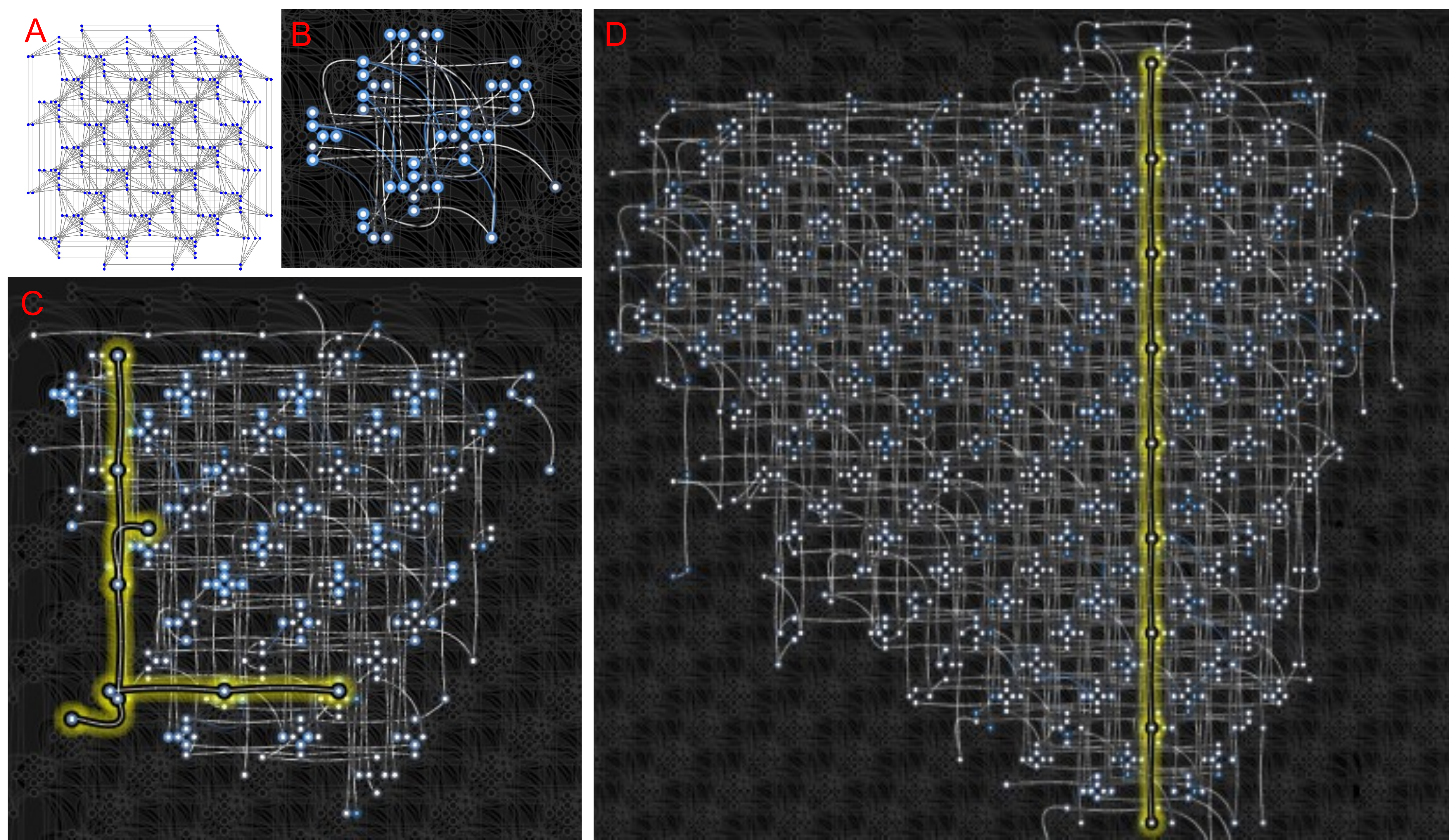} 
    \caption{(A) Pegasus topology structure; (B)–(D) Embedding visualization on QPUs (Pegasus topology) for MRA of problem sizes $N = 3, 5$ and $10$. 
    The corresponding longest chain for a logical qubit during the embedding is highlighted, respectively.
    }
    \label{embedding_structure} 
\end{figure*}

\section{Complementary Comparison with Classical Local Solvers} \label{s:local_solvers}

\noindent While both Levenberg-Marquardt~(LM) and trust-region methods leverage iterative frameworks analogous to IQARS, their implementations fundamentally differ in handling the $\mathrm{SO}(3)$ manifold's non-convexity.
We give further explanations as a supplement.
LM approximates second-order behavior through a damped Gauss-Newton scheme with Hessian approximation, while trust-region methods solve constrained quadratic subproblems within a radius.
Crucially, both approaches consistently convexify the $\mathrm{SO}(3)$ Riemannian geometry through Euclidean projections and can generate artificial critical points.
In contrast, IQARS preserves geometric integrity without convexification through: (1) exact $\mathrm{SO}(3)$ constraints via Prop.~\ref{proposition1}'s exponential maps, (2) adaptive trust-region optimization and orthogonality-promoting penalty, and (3) quantum annealing's Hamiltonian evolution enabling efficient exploration of a non-convex landscape via tunneling.
We perform additional quantitative empirical comparison with the mentioned solvers on noisy datasets with $N = 20$ and record the results in Tab.~\ref{Comparison_results}; the results confirm that IQARS maintains a clear advantage and further validates the theoretical claims.

\begin{table}[]
\scalebox{1.05}{
\begin{tabular}{@{}c|cccc@{}}
\toprule
\multicolumn{1}{c|}{Noise Levels} & $\pi/10 $ & $\pi/5 $ & $\pi/3 $ & $\pi/2 $ \\ \midrule
LM           & 0.2209                  & 0.4593                 & 0.7443                 & 1.1179                 \\
Trust-Region & 0.2341                  & 0.4793                 & 0.7839                 & 1.180                  \\
Ours         & \textbf{0.1929}                  & \textbf{0.3932}                 & \textbf{0.6388}                 & \textbf{0.9235}                 \\ \bottomrule
\end{tabular}
}
\caption{Performance comparison with local solvers on a synthetic noisy dataset across different noise levels $\sigma$.}
\label{Comparison_results}
\vspace{-15pt}
\end{table}

\section{Alternative Initialization Strategy} \label{s:init}

\noindent Our IQARS is configured with identity initialization by default to provide a neural and unbiased way for MRA and establish a baseline configuration free of initialization bias.
Notably, it also naturally accommodates alternative initialization schemes such as minimum spanning tree~(MST)~\cite{hartley2013rotation} which was originally integrated into L1-IRLS.
MST initialization provides a principled approach for generating initial rotation estimates that are geometrically consistent with relative measurements.
This especially benefits iterative algorithms such as ours under limited quantum resource budget.
We replace the original identity initialization protocol in IQARS with MST, and perform additional evaluations.
Other settings and configurations remain consistent as in primary experiments.
Quantitative results, as presented in Tab.~\ref{t:spanning_tree}, show that MST initialization under our configuration can indeed lead to significantly improved MRA performance.

\begin{table}[]
\centering
\begin{tabular}{@{}c|cccc@{}}
\toprule
Noise Ratios & $\pi/10 $ & $\pi/5 $ & $\pi/3 $ & $\pi/2 $ \\ \midrule
Ours (Identity init) & 0.193                     &   0.393                   &    0.639                  &     0.923                 \\
Ours (MST)        &   \textbf{0.177}                    &  \textbf{0.352}                    &   \textbf{0.573}                   &   \textbf{0.845}                   \\ \bottomrule
\end{tabular}
\caption{IQARS performance under different initialization schemes on synthetic noisy rotations of $N = 20$.}
\label{t:spanning_tree}
\vspace{-10pt}
\end{table}

\section{Quantum vs. Simulated Annealing Performance} \label{QA vs. SA}

\begin{figure}[t]
    \centering 
    \includegraphics[width = 0.48\textwidth]{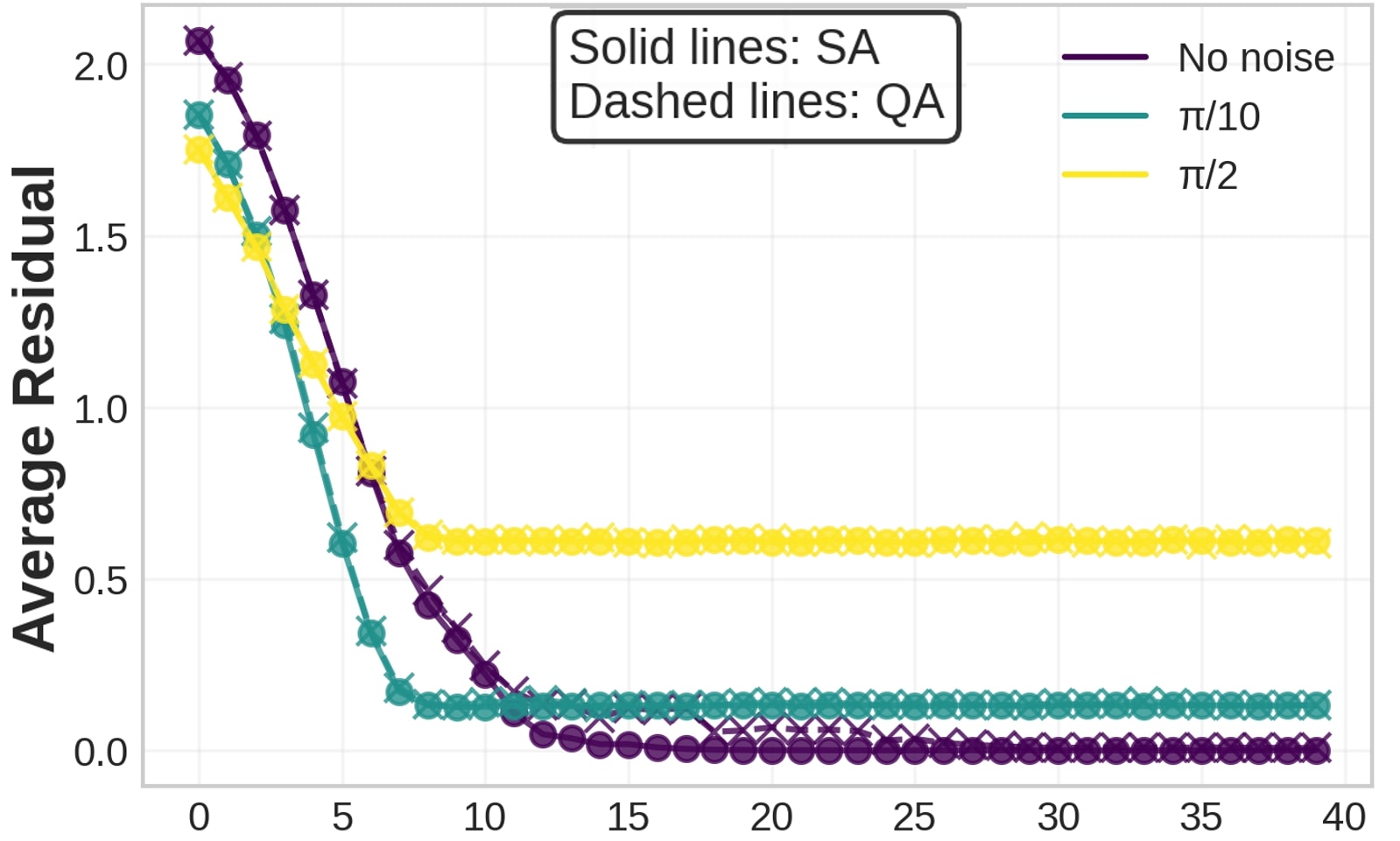} 
    \caption{MRA performance comparison between quantum and simulated annealing for the dataset across different noise levels.} 
    \label{sa_vs_qa}
    \vspace{-15pt}
\end{figure}

Besides quantum annealers, D-Wave also provides a classical simulated annealing (SA) solver.
It serves as a classical benchmark for evaluating quantum annealing performance, particularly in assessing the hardware noise effect on the solution quality.
SA, inspired by thermodynamic annealing in metallurgy, utilizes thermal fluctuations governed by the Metropolis-Hastings algorithm to escape local minima with state transitions following the Boltzmann acceptance criterion.
%
It operates classically, requiring \(\mathcal{O}(n^2)\) operations per iteration for \(n\)-variable QUBO problems due to pairwise interaction calculations.
In contrast, while the theoretical time complexity of quantum annealing is governed by the adiabatic theorem, annealing time in practice is by default set to $20\mu s$ on D-Wave machines; see set-up in Sec.~\ref{sec:experiments}.
We performed an empirical comparison with MRA on the dataset across various noise levels, and benchmarked QA against the SA solver; results are visualized in Fig.~\ref{sa_vs_qa}.
It is noticeable that SA exhibits marginally better convergence than quantum annealing executed on real noisy machines for a noiseless dataset.
With increasing noise levels, quantum annealers can achieve performance comparable to classical SA implementations.
We also record the clock time statistics.
During execution, SA takes around 1.32s while QA, including data transmission overhead, takes 0.38s per iteration in the experiment.
Considering SA's quadratic scaling, it is expected that the computational demand for SA becomes more pronounced for larger problem sizes.
These findings suggest that while QA algorithms leveraging practical quantum machines yet may still underperform SA in terms of performance, the inherent time efficiency and scalability suggest a compelling pathway for addressing computationally hard optimization problems in the future.

\section{Additional Visualizations} \label{s:additional}

We provide additional visualizations to complement the main text.
Fig.~\ref{q_resources} visualizes the physical qubit requirements for embedding IQARS QUBO instances on different generations of D-Wave annealers across different problem sizes $N$.

\begin{figure}[h]
    \centering 
    \includegraphics[width = 0.48\textwidth]{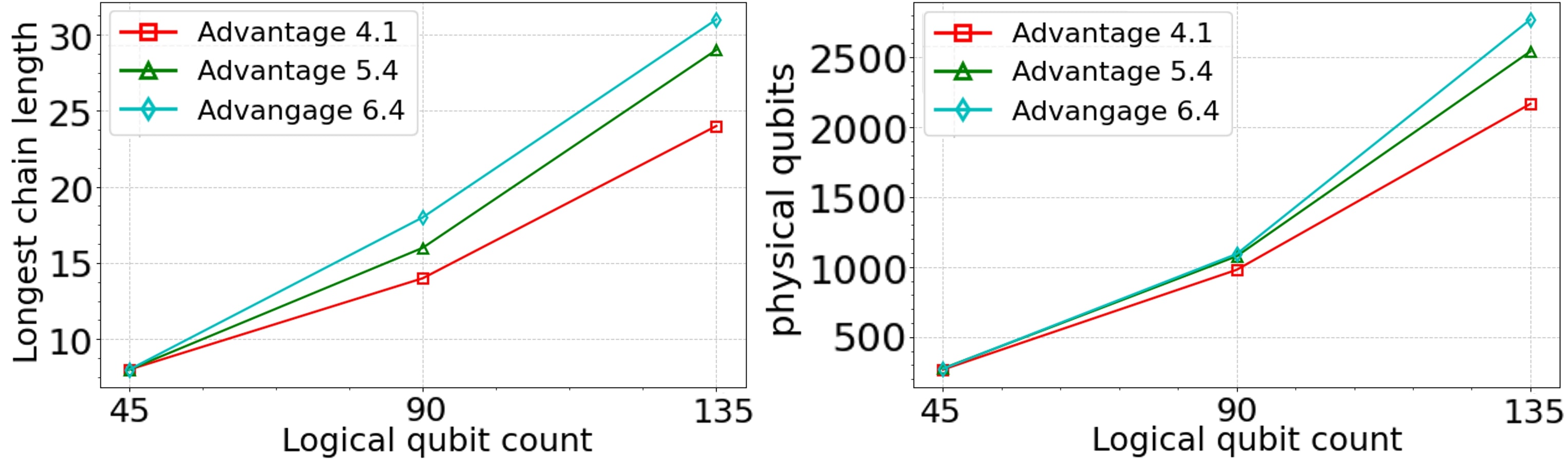} 
    \caption{Physical hardware resource occupation for logical problems of various sizes.}
    \label{q_resources}
    \vspace{-15pt}
\end{figure}